%% file: FolT_MCMC_main.tex
\documentclass[a4paper,12pt]{article}

\usepackage[a4paper,margin=2.5cm]{geometry}

\usepackage{amsmath,amssymb,amsthm}
\usepackage{mathtools}
\usepackage{bm}
\usepackage{enumitem}
\usepackage{booktabs}
\usepackage{natbib}
\usepackage{graphicx}
\graphicspath{{figures/}}
\usepackage{placeins}
\usepackage[colorlinks,citecolor=blue,linkcolor=blue,urlcolor=blue]{hyperref}
\usepackage{url}
\usepackage{setspace}
\newtheorem{theorem}{Theorem}[section]
\newtheorem{lemma}[theorem]{Lemma}
\newtheorem{proposition}[theorem]{Proposition}
\newtheorem{corollary}[theorem]{Corollary}
\newtheorem{assumption}[theorem]{Assumption}
\theoremstyle{definition}
\newtheorem{definition}[theorem]{Definition}
\theoremstyle{remark}
\newtheorem{remark}[theorem]{Remark}

\newcommand{\R}{\mathbb{R}}
\newcommand{\E}{\mathbb{E}}
\newcommand{\Prob}{\mathbb{P}}

\newcommand{\osc}{\mathrm{osc}}

\newcommand{\diam}{\mathrm{diam}}
\newcommand{\Vol}{\mathrm{Vol}}

\newcommand{\cK}{\mathcal{K}}
\newcommand{\cN}{\mathcal{N}}

\title{Folded Transport MCMC: Eliminating Label Switching
  by Sampling on a Fundamental Domain}

\author{Jun Hu \\
  {\small Wuhan University of Technology} \\
  {\small \texttt{junhu22@whut.edu.cn}}}

\date{}

\begin{document}
\maketitle

\begin{abstract}
In Bayesian mixture models and other exchangeable-component
models, the posterior is invariant under permutation of
component labels, creating $m!$ equivalent modes---a
phenomenon known as label switching.  Standard MCMC methods
either mix poorly across these modes or rely on post-hoc
relabelling that cannot guarantee the sampler has converged.
We propose \emph{Folded Transport MCMC} (FolT-MCMC), which
eliminates label switching before sampling by restricting the
Markov chain to a \emph{fundamental domain}---a sorted or
reflected subspace containing exactly one representative from
each symmetric mode.  The proposal distribution is a learned
normalising flow whose density is symmetrised over the group
orbits, ensuring correct targeting on the reduced space.
We show that this construction preserves a computable
convergence diagnostic based on the oscillation of the
log-density ratio, and that this diagnostic becomes sharper on the fundamental
domain whenever the original-space flow under-covers one or
more symmetric modes.
Experiments on Gaussian mixtures ($d = 2$--$20$),
label-switching targets (up to $24$ equivalent modes), a
standard Bayesian three-component mixture posterior, and real
accelerometer data from a supertall building demonstrate
improvement ratios of $2\times$ to $145\times$ in the
convergence diagnostic, with the folded diagnostic remaining
stable across dimensions while the unfolded diagnostic
collapses.
\end{abstract}


\bigskip

\input{section1_introduction}
\input{section2_background}
\input{section3_framework}
\input{section4_theory}
\input{section5_experiments}
\input{section6_application}
\input{section7_discussion}

\bibliography{references}

\newpage
\input{supplement}

\end{document}

%% file: section1_introduction.tex

\section{Introduction}
\label{sec:introduction}

\subsection{The label-switching problem}

Bayesian mixture models with $m$ exchangeable components define
posterior distributions with $m!$ equivalent modes: for every
posterior sample $(\mu_1, \sigma_1, \ldots, \mu_m, \sigma_m)$,
every permutation of the component indices
$(\mu_{\tau(1)}, \sigma_{\tau(1)}, \ldots)$ has exactly the
same posterior density.  This phenomenon, known as
\emph{label switching} \citep{Stephens2000,Jasra2005,Celeux2000},
is a central challenge in Bayesian computation for mixture models,
hidden Markov models, factor models, and structural identification
problems with closely-spaced modes
\citep{Au2011,YangLamBeck2019}.

The standard response is \emph{post-hoc relabelling}: run MCMC
on the full $m!$-modal posterior, then reorder the samples after
the fact using a loss function \citep{Stephens2000}, an
equivalence-class algorithm \citep{Papastamoulis2010}, or random
permutations at each step \citep{FruhwirthSchnatter2001}.
These methods improve the interpretability of the output
but do not address the fundamental difficulty: the Markov chain
must still traverse $m!$ symmetric modes, and there is no
guarantee that it has done so adequately.

\subsection{Our approach: sampling on a fundamental domain}

A cleaner alternative is to sample directly on a
\emph{fundamental domain}---a subspace that contains exactly
one representative from each group of equivalent modes.
For a three-component mixture, sorting the components by their
first parameter ($\mu_1 \leq \mu_2 \leq \mu_3$) maps the
$3! = 6$ equivalent modes to a single sorted mode on the
``sorted chamber.''  The posterior restricted to this domain,
which we call the \emph{folded target}, integrates to one
and retains all the inferential content of the original posterior
without the redundant copies.

This idea has appeared informally in the mixture literature
\citep{Celeux2000,FruhwirthSchnatter2001}, but has not previously
been combined with learned proposal distributions that can
efficiently explore the folded target.  We introduce
\emph{Folded Transport MCMC} (FolT-MCMC), which pairs the
fundamental-domain construction with a \emph{normalising flow}---a
learnable, invertible transformation that maps a simple
distribution (e.g., a standard normal) to an approximation of
the target.  Normalising flows have been used as proposal
distributions for Metropolis--Hastings samplers
\citep{Parno2018,Hoffman2019,Gabrie2022}; FolT-MCMC trains the
flow on the folded target and symmetrises its density over the
group orbits to form a \emph{quotient proposal} that covers
the fundamental domain.

\subsection{A computable convergence diagnostic}

A distinctive feature of FolT-MCMC is that the folded sampler
comes with a computable \emph{convergence diagnostic} based
on the oscillation of the log-density ratio between target and
proposal.

For an independence Metropolis--Hastings sampler with target
density $\pi$ and proposal density $q$, the spectral gap
$\gamma$ (which controls the geometric convergence rate) is
bounded below by
\[
  \gamma \;\geq\; \frac{2}{1 + \exp(\mathrm{osc}(h))},
  \qquad
  h = \log\pi - \log q,
\]
where $\mathrm{osc}(h) = \sup h - \inf h$ is the oscillation
of the log-density ratio \citep{MengersenTweedie1996}.
When the flow perfectly matches the target, the ratio is
constant and $\gamma = 1$; as the approximation degrades,
the oscillation grows and $\gamma$ shrinks exponentially.

Computing this bound from finite samples requires controlling
the oscillation over a high-probability region of the posterior
(a credible set).  Recent work by the author has developed a
framework for doing so: a coverage-based empirical oscillation
bound combined with Lipschitz interpolation yields a computable
upper bound on $\mathrm{osc}(h)$ from posterior samples alone
\citep{LCNF,CerTMCMC}.  We call the resulting quantity a
\emph{convergence diagnostic} rather than a spectral-gap
certificate, because it controls the oscillation on the
high-probability core of the posterior but not on the tails
(see Section~\ref{sec:theory} for a precise statement).

The key observation motivating this paper is that
\emph{label switching destroys this diagnostic}.  When the
posterior has $m!$ symmetric modes and the flow can only
approximate one of them, the log-density ratio is large on the
off-modes, inflating the oscillation and rendering the
diagnostic vacuous (i.e., the bound gives $\gamma \approx 0$
even when the sampler is working well empirically).  By folding
onto the fundamental domain, FolT-MCMC removes the cross-mode
oscillation, restoring a meaningful diagnostic.

\subsection{Contributions}

The three main contributions are:
\begin{enumerate}[nosep]
  \item \textbf{Fundamental-domain sampler.}
  For a posterior $\pi$ invariant under a finite group $G$
  (e.g., $G = S_m$ for $m$-component mixtures), we define
  the folded target on the fundamental domain, prove that its
  credible sets are simply the intersection of the original
  credible sets with the domain, and construct a symmetrised
  proposal from a learned normalising flow
  (Section~\ref{sec:framework}).

  \item \textbf{Sharper convergence diagnostic.}
  We show that the oscillation-based diagnostic transfers
  to the fundamental domain with improved inputs: the
  credible-set diameter does not increase, the density floor
  improves by a factor of $|G|$, and the covering radius
  decreases.  When the flow trained on the original space
  under-covers one or more symmetric modes (as observed in our
  experiments), the folded diagnostic is sharper
  (Section~\ref{sec:theory}).

  \item \textbf{Empirical validation.}
  On Gaussian mixtures ($d = 2$--$20$), label-switching
  targets (up to $24$ equivalent modes), a standard Bayesian
  three-component mixture posterior, and real accelerometer
  data from a supertall building during Typhoon Mangkhut,
  the diagnostic improvement ratio ranges from $2\times$ to
  $145\times$.  The folded diagnostic is empirically stable
  across dimensions, while the unfolded diagnostic collapses.
  On the Bayesian mixture benchmark, post-hoc relabelling
  recovers correct point estimates but provides no convergence
  guarantee, demonstrating that quotient-space sampling is
  needed for diagnosing posterior convergence
  (Sections~\ref{sec:experiments}--\ref{sec:application}).
\end{enumerate}

\subsection{Outline}

Section~\ref{sec:background} provides background on normalising
flows as MCMC proposals, the oscillation-based convergence
diagnostic, and symmetry in Bayesian models.
Section~\ref{sec:framework} develops the FolT-MCMC framework.
Section~\ref{sec:theory} establishes the theoretical guarantees.
Section~\ref{sec:experiments} presents the experiments.
Section~\ref{sec:application} applies FolT-MCMC to structural
modal identification with real typhoon data.
Section~\ref{sec:discussion} discusses limitations and extensions.

\paragraph{Terminology and notation.}
Throughout, we use the following terms:
\emph{normalising flow}---a learnable invertible map
  $T\colon \mathbb{R}^d \to \mathbb{R}^d$ that transforms a
  simple base distribution to an approximation of the target;
\emph{independence Metropolis--Hastings (IMH)}---a
  Metropolis--Hastings sampler whose proposals are drawn
  independently of the current state;
\emph{spectral gap} ($\gamma$)---a number in $[0,1]$ measuring
  the geometric convergence rate of a Markov chain;
\emph{oscillation} ($\mathrm{osc}(h)$)---the range
  $\sup h - \inf h$ of the log-density ratio
  $h = \log\pi - \log q$;
\emph{fundamental domain} ($D$)---a subspace containing one
  representative per symmetric mode (e.g., the sorted chamber);
\emph{folded target} ($\pi_F$)---the posterior restricted and
  renormalised to the fundamental domain;
\emph{convergence diagnostic}---the computable lower bound on
  $\gamma$ obtained from finite samples via the oscillation
  framework.
All acronyms are defined at first use.

%% file: section2_background.tex

\section{Background and notation}
\label{sec:background}

\subsection{Transport MCMC and independence Metropolis--Hastings}
\label{sec:bg_transport}

A normalising flow $T\colon \R^d \to \R^d$ is a learnable
diffeomorphism that pushes a simple base density~$q_0$
(typically $\cN(0, I_d)$) to an approximation
$q = T_{\#}q_0$ of a target density~$\pi$.
In transport MCMC, the flow serves as the proposal mechanism
for a Markov chain whose stationary distribution is
exactly~$\pi$, correcting any approximation error in~$q$
through the Metropolis--Hastings acceptance step
\citep{Parno2018,Hoffman2019,Gabrie2022}.

The simplest such chain is \emph{independence
Metropolis--Hastings} (IMH): given the current state~$\theta$,
propose $\theta' \sim q$ and accept with probability
$\min(1,\, e^{h(\theta) - h(\theta')})$, where
$h = \log\pi - \log q$ is the log-density ratio.
IMH is particularly well suited to transport proposals because
the proposal is cheap to evaluate (a single forward pass through
the flow) and the acceptance probability depends only on the
quality of the density-ratio fit.

The spectral gap~$\gamma$ of the IMH chain---measuring the
rate of geometric convergence---is bounded below by the
Mengersen--Tweedie inequality \citep{MengersenTweedie1996}:
\begin{equation}\label{eq:mt_bound}
  \gamma
  \;\geq\;
  \frac{2}{1 + \exp\!\bigl(\osc(h)\bigr)},
  \qquad
  \osc(h)
  := \sup_\theta h(\theta) - \inf_\theta h(\theta).
\end{equation}
When the flow approximation is perfect ($q = \pi$),
$\osc(h) = 0$ and $\gamma = 1$; as the approximation degrades,
$\osc(h)$ grows and $\gamma$ shrinks exponentially.

\subsection{A computable convergence diagnostic}
\label{sec:bg_lcnf}

The oscillation $\osc(h)$ is not directly computable from
finite samples, because both the supremum and infimum range
over the entire support of~$\pi$.  Recent work by the author
\citep{LCNF} provides a computable upper bound on the
oscillation over a highest posterior density (HPD) credible
set~$\cK_\alpha$---the smallest region containing
$(1{-}\alpha)$ of the posterior mass---yielding a
computable core lower-bound diagnostic, using three ingredients.

\paragraph{Spectrally normalised RealNVP.}
The flow~$T$ is a RealNVP---a coupling-based normalising flow
architecture \citep{DinhSohlDickstein2017} in which each layer transforms
half of the coordinates conditionally on the other half---built
from $L$~coupling layers.  Each conditioner network is
\emph{spectrally normalised}: a weight-matrix constraint that
bounds the Lipschitz constant of the network
\citep{Miyato2018}, here combined with a scale clip~$c$.
This gives a per-layer bi-Lipschitz bound, ensuring that the
Jacobian and its log-determinant are well controlled.

\paragraph{Empirical oscillation theorem.}
Given $n$~i.i.d.\ samples $\{Z_i\}$ from~$\pi$ and the
$(1{-}\alpha)$-HPD set $\cK_\alpha$, the empirical oscillation
theorem of \citet{LCNF} states
that with probability $\geq 1 - \delta$:
\begin{equation}\label{eq:lcnf_thm}
  \osc_{\cK_\alpha}(h)
  \;\leq\;
  \underbrace{\max_i h(Z_i) - \min_i h(Z_i)}_{\widehat{\osc}_n}
  \;+\;
  \underbrace{2\,M\,\varepsilon^*}_{\text{interpolation error}},
\end{equation}
where $M = \sup_{\cK_\alpha}\|\nabla h\|$ is the local gradient
constant and $\varepsilon^*$ is the minimal covering radius
satisfying a covering condition that depends on
$\diam(\cK_\alpha)$, $\pi_{\min}$, and~$n$.
The right-hand side is fully computable from samples and
network parameters.

\paragraph{Quantile-core diagnostic.}
The HPD-core oscillation bound~\eqref{eq:lcnf_thm}
can still be dominated by extreme density-ratio values
among the retained core samples.  The quantile-core extension
of \citet{CerTMCMC} replaces the full HPD set with a
$\rho$-trimmed core that excludes the highest-oscillation
fraction, yielding a tighter \emph{quantile-core} (QC) diagnostic.
In all experiments of the present paper, we report the
quantile-core diagnostic at $\rho = 0.05$.

\paragraph{Limitation: multimodality.}
When $\pi$ has $k$~well-separated modes and the flow can
only Gaussianise one of them, the log-density ratio~$h$
is large on the off-modes (where $q \ll \pi$), inflating
$\widehat{\osc}_n$ by a cross-mode gap that grows with
mode separation and dimension.
This is the fundamental barrier that FolT-MCMC is designed
to overcome.

\subsection{Symmetry in Bayesian inference}
\label{sec:bg_symmetry}

Many Bayesian posteriors possess a finite symmetry group~$G$
under which the likelihood and prior are jointly invariant.
The resulting symmetric multimodality is a well-known
obstacle to MCMC mixing and inference.

\paragraph{Label switching.}
In $m$-component mixture models, the likelihood is invariant
under permutation of component labels ($G = S_m$,
$|G| = m!$).  This creates $m!$~equivalent posterior
modes, between which standard MCMC chains mix poorly
\citep{Celeux2000,Jasra2005}.  Existing solutions are
predominantly \emph{post-hoc}: after running MCMC in
the full $m!$-modal space, the samples are relabelled
using a loss function \citep{Stephens2000}, an
equivalence-class representative algorithm
\citep{Papastamoulis2010}, or random permutation at each
step \citep{FruhwirthSchnatter2001}.  These methods do not
address the poor mixing of the underlying chain, nor do they
provide a convergence diagnostic.

\paragraph{Structural modal identification.}
In Bayesian operational modal analysis of buildings and bridges,
closely-spaced structural modes create approximate label-switching
symmetry: the posterior over $m$~sets of modal parameters
(frequency, damping ratio) is nearly $S_m$-invariant when the
modes have similar frequencies \citep{Au2011,YangLamBeck2019}.
This is a key motivation for the present work.

\paragraph{Equivariant normalising flows.}
An alternative approach constructs flows that respect the
symmetry by design: $T(g \cdot z) = g \cdot T(z)$
\citep{Kohler2020,Rezende2019}.  Such equivariant flows reduce
parameter redundancy but must still represent all $|G|$~modes.
FolT-MCMC takes the complementary approach of
\emph{eliminating} the symmetric modes before sampling,
reducing the target to a single representative per orbit.
The two strategies are orthogonal and could in principle be
combined.

%% file: section3_framework.tex

\section{The FolT-MCMC framework}
\label{sec:framework}

Let $\pi$ be a target density on~$\R^d$ that is invariant under a
finite group~$G$ acting on~$\R^d$: $\pi(g \cdot \theta) = \pi(\theta)$
for every $g \in G$ and almost every~$\theta$.  Write $s = |G|$.
Typical examples are label-switching symmetry in mixture models
($G = S_m$, the symmetric group on $m$~components, $s = m!$) and
reflection symmetry ($G = \mathbb{Z}_2$, $s = 2$).

The central idea of FolT-MCMC is to \emph{eliminate} the symmetric
multimodality of~$\pi$ before sampling.  We restrict the target
to a fundamental domain, obtaining the folded target introduced
in Section~\ref{sec:introduction}.  We then run an independence
Metropolis--Hastings sampler on this domain, using a learned
normalising flow as proposal, and compute the convergence
diagnostic from the oscillation of the log-density ratio.

\subsection{The folded target}
\label{sec:quotient_target}

A \emph{fundamental domain} for~$G$ is a measurable set
$D \subset \R^d$ such that $\R^d = \bigcup_{g \in G} g \cdot D$
with pairwise disjoint interiors and
$\pi(g \cdot D \cap g' \cdot D) = 0$ for $g \neq g'$.

\begin{definition}[Folded target]\label{def:quotient_target}
  The \emph{folded} or \emph{quotient target} on~$D$ is
  \begin{equation}\label{eq:pi_F}
    \pi_F(z) \;=\; s \cdot \pi(z), \qquad z \in D.
  \end{equation}
  The corresponding potential is
  $U_F(z) = -\log\pi_F(z) = U(z) - \log s$.
\end{definition}

\noindent
Normalisation is immediate: by the fundamental-domain decomposition,
$\int_D \pi_F = s \int_D \pi = s \cdot (1/s) = 1$.
The density $\pi_F$ retains one representative from each $G$-orbit
of modes, so the symmetric multimodality induced by the group action
is absent; non-symmetric multimodality, if present, may remain.

The next result shows that folding preserves the highest posterior
density (HPD) structure exactly.

\begin{proposition}[HPD identity]\label{prop:hpd}
  Let $u_\alpha$ be the $(1{-}\alpha)$-quantile of $U(\Theta)$ under
  $\Theta \sim \pi$, defining the HPD set
  $\cK_\alpha = \{\theta \in \R^d : U(\theta) \leq u_\alpha\}$.
  Then:
  \begin{enumerate}[nosep]
    \item $u_\alpha^F = u_\alpha - \log s$;
    \item $\cK_\alpha^F
      := \{z \in D : U_F(z) \leq u_\alpha^F\}
      = \cK_\alpha \cap D$ \;(a.e.).
  \end{enumerate}
\end{proposition}

\begin{proof}
  For any Borel function $\varphi\colon \R \to \R$, the $G$-invariance
  of~$\pi$ and the fundamental-domain decomposition give
  \[
    \int_{\R^d} \varphi(U(\theta))\,\pi(\theta)\,\mathrm{d}\theta
    = \sum_{g \in G} \int_{gD} \varphi(U(\theta))\,\pi(\theta)\,\mathrm{d}\theta
    = s \int_D \varphi(U(z))\,\pi(z)\,\mathrm{d}z
    = \int_D \varphi(U(z))\,\pi_F(z)\,\mathrm{d}z.
  \]
  Hence $U(Z_F)$ under $Z_F \sim \pi_F$ has the same distribution as
  $U(\Theta)$ under $\Theta \sim \pi$, so the two quantiles coincide:
  the $(1{-}\alpha)$-quantile of $U_F(Z_F) = U(Z_F) - \log s$
  is $u_\alpha - \log s$.
  Part~(ii) follows directly:
  $\{U_F \leq u_\alpha^F\} = \{U - \log s \leq u_\alpha - \log s\}
  = \{U \leq u_\alpha\}$.
\end{proof}

\begin{corollary}\label{cor:hpd_geometry}
  (a) $\diam(\cK_\alpha^F) \leq \diam(\cK_\alpha)$.
  (b) The density floor satisfies
  $\pi_{\min}^F := \inf_{\cK_\alpha^F} \pi_F
  = s \cdot \inf_{\cK_\alpha \cap D} \pi \geq s \cdot \pi_{\min}$,
  where $\pi_{\min} := \inf_{\cK_\alpha} \pi$.
\end{corollary}

\noindent
Using essential infima with respect to the Lebesgue measure,
the $G$-invariance of~$\pi$ ensures
$\operatorname{ess\,inf}_{\cK_\alpha \cap D} \pi
= \operatorname{ess\,inf}_{\cK_\alpha} \pi$,
so $\pi_{\min}^F = s\,\pi_{\min}$.  With ordinary infima this
equality holds for closed fundamental domains; in general we use
the conservative bound $\pi_{\min}^F \geq s\,\pi_{\min}$.
For label switching with $m = 4$ components, $s = 24$ and the
floor is $24$~times higher.

\subsection{Quotient proposal}
\label{sec:quotient_proposal}

Let $T_F\colon \R^d \to \R^d$ be a normalising flow
(e.g.\ spectrally normalised RealNVP) trained on folded samples
projected onto~$D$, and let
$q_{T_F} = (T_F)_{\#}\cN(0, I_d)$ denote the induced density
on~$\R^d$.

\begin{definition}[Folded (quotient) proposal]\label{def:quotient_proposal}
  The \emph{quotient proposal} on~$D$ is
  \begin{equation}\label{eq:q_F}
    q_F(z) \;=\; \sum_{g \in G} q_{T_F}(g \cdot z), \qquad z \in D.
  \end{equation}
\end{definition}

\noindent
Normalisation follows from the fundamental-domain decomposition:
$\int_D q_F = \sum_g \int_D q_{T_F}(g \cdot z)\,\mathrm{d}z
= \sum_g \int_{gD} q_{T_F}(\theta)\,\mathrm{d}\theta
= \int_{\R^d} q_{T_F} = 1$.
Intuitively, $q_F$ is the unique $G$-symmetric density on~$\R^d$
obtained by summing $q_{T_F}$ over the group orbits;
restricted to~$D$, it gives the proposal used by the sampler.

The log-density ratio governing the Metropolis--Hastings
acceptance is
\begin{equation}\label{eq:h_F}
  h_F(z)
  \;=\; \log\frac{\pi_F(z)}{q_F(z)}
  \;=\; \log\bigl(s \cdot \pi(z)\bigr)
  \;-\; \log\!\Bigl(\sum_{g \in G} q_{T_F}(g \cdot z)\Bigr).
\end{equation}

\begin{remark}[Relation to unfolded MH]\label{rem:not_equivalent}
  Independence MH on~$D$ with target~$\pi_F$ and proposal~$q_F$ is
  a valid Markov chain with stationary distribution~$\pi_F$.  It is
  \emph{not} identical to running MH on~$\R^d$ and projecting
  samples onto~$D$, because $\pi_F(z)/q_F(z)$ generally differs from
  $\pi(z)/q_{T_F}(z)$.  The two coincide when $q_{T_F}$ is itself
  $G$-invariant.  In practice, when $q_{T_F}$ concentrates on a single
  mode, the off-mode terms $q_{T_F}(g \cdot z)$ for $g \neq e$ are
  negligible for $z$ deep in~$D$, and the two ratios agree
  approximately.
\end{remark}

\subsection{Algorithm}
\label{sec:algorithm}

\begin{figure}[t]
\centering
\includegraphics[width=\textwidth]{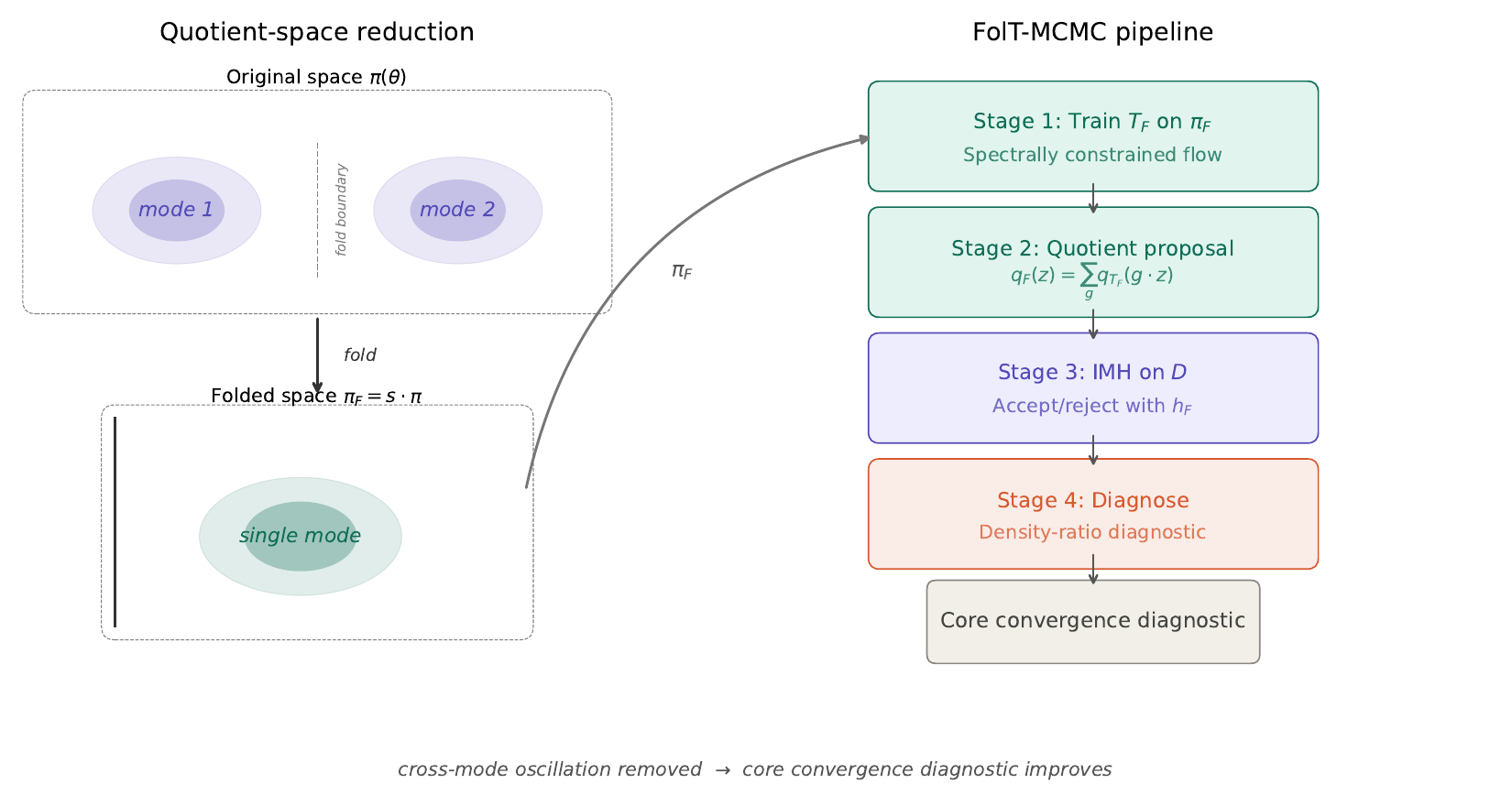}
\caption{FolT-MCMC pipeline.
  Left: the original posterior $\pi$ has $s$ symmetric modes;
  folding onto the fundamental domain~$D$ produces the quotient
  target $\pi_F$ with a single representative mode.
  Right: the four-stage pipeline trains a transport map on $\pi_F$,
  forms the quotient proposal, runs independence MH on~$D$, and
  computes the convergence diagnostic from the oscillation
  framework.}
\label{fig:schematic}
\end{figure}

FolT-MCMC combines folding, transport, sampling, and diagnosis
in the workflow below.

\medskip
\begin{center}
\fbox{\parbox{0.92\textwidth}{%
\setstretch{1.0}\small
\textbf{Algorithm: FolT-MCMC} \\[3pt]
\textit{Input:} target~$\pi$ with known symmetry group~$G$,
  fundamental domain~$D$. \\[2pt]
\textit{Stage~1 (Fold and train).} \\
\quad Draw reference samples $\{z_i\}_{i=1}^N$ from $\pi_F$
  (fold exact or MCMC samples into~$D$). \\
\quad Train a spectrally normalised RealNVP flow $T_F$ on~$\pi_F$
  using negative log-likelihood (NLL) + oscillation
  regularisation. \\
\quad Form the quotient proposal $q_F$ via~\eqref{eq:q_F}. \\[2pt]
\textit{Stage~2 (Sample).} \\
\quad Run independence Metropolis--Hastings on~$D$: \\
\quad\quad Propose $z' \sim q_F$ by sampling $z_0 \sim \cN(0,I)$,
  computing $\theta = T_F(z_0)$, setting $z' = \mathrm{proj}_D(\theta)$. \\
\quad\quad Accept with probability
  $\min\!\bigl(1,\,
  e^{\,h_F(z_{\mathrm{curr}}) - h_F(z')}\bigr)$. \\[2pt]
\textit{Stage~3 (Diagnose).} \\
\quad Compute the empirical oscillation $\widehat{\osc}_n^F$
  and gradient supremum~$\widehat{M}_n^F$
  from an independent diagnostic set on~$D$. \\
\quad Solve the covering condition~\eqref{eq:covering_cond_F}
  for $\varepsilon_F^*$. \\
\quad Report the quantile-core convergence diagnostic
  $\mathrm{QC}\,\underline{\gamma}_F
  \geq 2/(1 + \exp(\widehat{\osc}_n^F + 2\widehat{M}_n^F\varepsilon_F^*))$.
\\[2pt]
\textit{Output:} posterior samples on~$D$;
  quantile-core diagnostic~$\mathrm{QC}\,\underline{\gamma}_F$.
}}
\end{center}

\medskip
\noindent
We highlight two concrete instantiations.

\paragraph{Reflection fold ($G = \mathbb{Z}_2$).}
The non-identity element acts by negating the first coordinate:
$g \cdot \theta = (-\theta_1, \theta_2, \ldots, \theta_d)$.
The fundamental domain is $D = \{\theta_1 \geq 0\}$ and
$\mathrm{proj}_D(\theta) = (|\theta_1|, \theta_2, \ldots, \theta_d)$.
The quotient proposal reduces to
$q_F(z) = q_{T_F}(z) + q_{T_F}(-z_1, z_2, \ldots, z_d)$.

\paragraph{Permutation fold ($G = S_m$).}
Each $\theta \in \R^d$ is partitioned into $m$~blocks of $p$~parameters
($d = mp$), and $G = S_m$ permutes blocks.
The fundamental domain is the sorted chamber
$D = \{\theta : \theta^{(1)}_1 \leq \theta^{(2)}_1 \leq \cdots
\leq \theta^{(m)}_1\}$, where $\theta^{(j)}_1$ is the first parameter
(e.g.\ natural frequency) of the $j$th block.
Projection sorts the blocks by this parameter.

\subsection{Quotient metric and covering}
\label{sec:quotient_metric}

The empirical oscillation framework of \citet{LCNF} relies on a covering
lemma (Lemma~5.2 therein) that lower-bounds the probability mass
of balls intersected with the HPD set.  On the fundamental domain~$D$,
Euclidean balls near the fold boundary $\partial D$ may be truncated,
potentially losing mass.  We resolve this by working in the
\emph{quotient metric}.

\begin{definition}[Quotient metric]\label{def:quotient_metric}
  For $\theta, \theta' \in \R^d$, define
  \begin{equation}\label{eq:d_G}
    d_G(\theta, \theta')
    \;=\; \min_{g \in G}\, \|\theta - g \cdot \theta'\|.
  \end{equation}
  This descends to a well-defined metric on the orbit space $\R^d/G$.
\end{definition}

\noindent
Since $d_G \leq \|\cdot\|$ (take $g = e$), any Euclidean
$\varepsilon$-cover is automatically a quotient $\varepsilon$-cover,
giving the covering-number bound
\begin{equation}\label{eq:covering_number}
  \cN_G(\cK_\alpha^F, \varepsilon)
  \;\leq\;
  \cN_E(\cK_\alpha^F, \varepsilon)
  \;\leq\;
  \Bigl(\frac{2\diam(\cK_\alpha^F)}{\varepsilon} + 1\Bigr)^{\!d}.
\end{equation}

The advantage of the quotient metric emerges when bounding the
\emph{ball mass}.  For a point $z \in D$ whose stabiliser is trivial
($\mathrm{Stab}(z) = \{e\}$), the quotient ball
$B_G(z, \varepsilon) := \{z' : d_G(z, z') \leq \varepsilon\}$
corresponds in $\R^d$ to the union
$\bigcup_{g \in G}\bigl(B(g\!\cdot\!z,\,\varepsilon) \cap g\!\cdot\!D\bigr)$,
which has total volume equal to that of a full Euclidean ball
$V_d\,\varepsilon^d$---no half-ball or boundary correction is
needed.  For well-separated mixture posteriors, the HPD set
$\cK_\alpha^F$ lies in the interior of~$D$ and all stabilisers are
trivial; we formalise this as a regularity condition.

\begin{assumption}[Generic stabiliser]\label{asm:generic_stab}
  $|\mathrm{Stab}(z)| = 1$ for every $z \in \cK_\alpha^F$.
\end{assumption}

\noindent
Assumption~\ref{asm:generic_stab} holds whenever the mode centres are
distinct after projection into~$D$ and the modes are sufficiently
well separated that the HPD set does not touch the fold boundary.
When it fails (e.g.\ at fixed-point strata of the group action),
the ball-mass bound acquires a correction factor
$1/h_{\max}$ where
$h_{\max} = \max_{z \in \cK_\alpha^F} |\mathrm{Stab}(z)|$;
see Supplement~\ref{app:stabiliser}.

The remaining ingredient is a Lipschitz property of $h_F$ under
the quotient metric.

\begin{lemma}[Quotient Lipschitz constant]\label{lem:lip_quotient}
  Let $h_F$ be extended to $\R^d$ by
  $h_F(\theta) = \log\bigl(s\pi(\theta)\bigr)
  - \log\bigl(\sum_g q_{T_F}(g\theta)\bigr)$,
  and let
  $M_F = \sup\bigl\{\|\nabla h_F(w)\| :
  d_G(w, \cK_\alpha^F) \leq \varepsilon_F^*\bigr\}$.
  Then for all $x, y \in \cK_\alpha^F$ with $d_G(x,y) \leq \varepsilon_F^*$:
  \begin{equation}\label{eq:lip_dG}
    |h_F(x) - h_F(y)| \;\leq\; M_F \cdot d_G(x, y).
  \end{equation}
\end{lemma}

\begin{proof}
  Fix $x, y \in \cK_\alpha^F$ with $d_G(x,y) \leq \varepsilon_F^*$.
  Let $g^\star \in G$ achieve
  $d_G(x,y) = \|x - g^\star \cdot y\|$.
  Since $d_G(x,y) \leq \varepsilon_F^*$, the segment from $x$
  to $g^\star \cdot y$ lies in the
  $\varepsilon_F^*$-neighbourhood of $\cK_\alpha^F$.
  By the mean value theorem,
  \[
    |h_F(x) - h_F(g^\star \cdot y)|
    \leq M_F \cdot \|x - g^\star \cdot y\|
    = M_F \cdot d_G(x,y).
  \]
  Using $h_F(g^\star \cdot y) = h_F(y)$ ($G$-invariance) completes
  the proof.
\end{proof}

\noindent
With the covering number~\eqref{eq:covering_number}, the ball-mass
bound, and Lemma~\ref{lem:lip_quotient}, the empirical
oscillation theorem of \citet[Theorem~5.1]{LCNF} carries over to the
quotient setting.  The covering condition becomes
\begin{equation}\label{eq:covering_cond_F}
  \Bigl(\frac{D_F}{\varepsilon} + 1\Bigr)^{\!d}
  \cdot
  \bigl(1 - (s/h_{\max})\,\pi_{\min}^D\, V_d\, \varepsilon^d\bigr)^{n_{\mathrm{eff}}}
  \;\leq\; \frac{\delta}{3},
\end{equation}
where $D_F = \diam(\cK_\alpha^F)$,
$\pi_{\min}^D = \inf_{\cK_\alpha \cap D}\pi$, and
$h_{\max} = \max_{z \in \cK_\alpha^F}|\mathrm{Stab}(z)|$.
Under Assumption~\ref{asm:generic_stab}, $h_{\max} = 1$ and the
mass factor simplifies to $\pi_{\min}^F \cdot V_d \cdot \varepsilon^d$.
The resulting quantile-core diagnostic is developed in
Section~\ref{sec:theory}.

\begin{remark}[Computational scalability]\label{rem:scalability}
  The orbit sum in~\eqref{eq:q_F} has $s = |G|$ terms.
  For $S_m$ with $m \leq 5$ ($s \leq 120$), this is inexpensive.
  For larger~$m$, the sum can be approximated by restricting to
  permutations that lie within a neighbourhood of the identity,
  since distant permutations contribute negligibly when the modes
  are well separated.
\end{remark}

%% file: section4_theory.tex

\section{Theoretical guarantees}
\label{sec:theory}

We now establish that the oscillation framework of \citet{LCNF}
transfers to the quotient setting of
Section~\ref{sec:framework}, and that the resulting
quantile-core convergence diagnostic is at least as sharp
as---and typically much sharper than---the bound obtained on
the original space.

\subsection{The folded convergence diagnostic}
\label{sec:folded_cert}

The empirical oscillation theorem of \citet[Theorem~5.1]{LCNF}
requires three regularity conditions on the target, the HPD set,
and the log-density ratio.  We verify each in turn for the quotient
setting $(\pi_F, \cK_\alpha^F, h_F)$.

\paragraph{Regularity condition 1 (Compact HPD set with positive density floor).}
By Proposition~\ref{prop:hpd}, $\cK_\alpha^F = \cK_\alpha \cap D$ is
compact (intersection of a compact set and a closed domain).
Its diameter satisfies $D_F \leq D_U := \diam(\cK_\alpha)$
(Corollary~\ref{cor:hpd_geometry}a), and its density floor satisfies
$\pi_{\min}^F \geq s\,\pi_{\min} > 0$
(Corollary~\ref{cor:hpd_geometry}b).

\paragraph{Regularity condition 2 (Smooth HPD boundary).}
The boundary $\partial\cK_\alpha^F$ consists of three parts:
(i)~the HPD level set $\{U_F = u_\alpha^F\} \cap \mathrm{int}(D)$,
which is a smooth $(d{-}1)$-dimensional manifold wherever
$\nabla U_F \neq 0$ (holding generically by Sard's theorem);
(ii)~the fold boundary $\partial D \cap \mathrm{int}(\cK_\alpha^F)$,
which inherits the smoothness of~$\partial D$ (a hyperplane for
reflections, a polyhedral cone for permutations); and
(iii)~the corner set $\{U_F = u_\alpha^F\} \cap \partial D$,
which has Hausdorff dimension at most $d{-}2$ and does not affect
the covering argument.
Under Assumption~\ref{asm:generic_stab}, parts (ii) and (iii)
are empty and $\partial\cK_\alpha^F$ is smooth.

\paragraph{Regularity condition 3 (Local regularity of $h_F$).}
On $\mathrm{int}(D)$, $\pi_F = s\pi$ inherits the smoothness
of~$\pi$ (typically $C^\infty$ for Gaussian mixtures), and
$q_F = \sum_g q_{T_F}(g\,\cdot\,)$ is $C^\infty$ when $T_F$
uses $\tanh$ activations.  Therefore $h_F = \log\pi_F - \log q_F$
is $C^1$ on a neighbourhood of $\cK_\alpha^F$, and the gradient
supremum $M_F$ (Definition in Lemma~\ref{lem:lip_quotient}) is
finite.

\medskip
With all three conditions verified, we state the folded diagnostic.

\begin{theorem}[FolT-MCMC convergence diagnostic]\label{thm:folded_cert}
  Suppose $\pi$ is $G$-invariant with $|G| = s$,
  Assumption~\ref{asm:generic_stab} holds, and $h_F$ is $C^1$ on a
  neighbourhood of $\cK_\alpha^F$.
  Draw $Z_1, \ldots, Z_n \overset{\mathrm{iid}}{\sim} \pi_F$ and
  let $\widehat{\osc}_n^F = \max_i h_F(Z_i) - \min_i h_F(Z_i)$.
  Let $\varepsilon_F^*$ be the smallest $\varepsilon > 0$
  satisfying~\eqref{eq:covering_cond_F} with $h_{\max} = 1$, and
  define $C_F = \widehat{\osc}_n^F + 2 M_F \varepsilon_F^*$.

  Then with probability at least $1 - \delta$:
  \begin{equation}\label{eq:osc_bound_F}
    \osc_{\cK_\alpha^F}(h_F)
    \;\leq\; C_F
    \;=\; \widehat{\osc}_n^F + 2\,M_F\,\varepsilon_F^*.
  \end{equation}
  The resulting $\alpha$-HPD convergence diagnostic is
  \begin{equation}\label{eq:gamma_F}
    \underline{\gamma}_{F,\alpha}
    \;:=\; \frac{2}{1 + \exp(C_F)}.
  \end{equation}
  We call $\underline{\gamma}_{F,\alpha}$ the HPD-core
  convergence diagnostic.
  In all experiments we report the quantile-core convergence
  diagnostic (denoted QC~$\underline{\gamma}$) at $\rho = 0.05$,
  following \citet{CerTMCMC}.
\end{theorem}

\begin{remark}[Scope of the diagnostic]\label{rem:scope}
  The bound~\eqref{eq:gamma_F} controls density-ratio stability
  on the HPD core $\cK_\alpha^F$.  It does not constitute a
  full-support spectral-gap certificate, which would require
  controlling $\sup_D h_F - \inf_D h_F$ including the tails where
  $\pi_F$ has negligible mass.  Full-support certification for
  transport MCMC with unbounded state spaces is an open problem
  shared with the frameworks of \citet{LCNF,CerTMCMC}.
  The contribution of the present paper is
  to restore non-vacuousness of the core diagnostic in the
  presence of symmetric multimodality, not to close the
  core-to-full-support gap.
\end{remark}

\begin{proof}
  Apply \citet[Theorem~5.1]{LCNF} with the quotient-metric covering
  (Section~\ref{sec:quotient_metric}): the covering
  number is bounded by~\eqref{eq:covering_number}, the ball mass by
  $\pi_{\min}^F \cdot V_d \cdot \varepsilon^d$
  (Assumption~\ref{asm:generic_stab}), and the interpolation error by
  $M_F \cdot d_G$ (Lemma~\ref{lem:lip_quotient}).
  Applying the Mengersen--Tweedie lower-bound functional
  \citep{MengersenTweedie1996} to the bounded HPD-core
  oscillation gives the core diagnostic~\eqref{eq:gamma_F}.
\end{proof}

\subsection{Covering radius improvement}
\label{sec:covering_improvement}

The folded diagnostic~\eqref{eq:osc_bound_F} has the same
\emph{form} as the unfolded one, but its three constituent
quantities are more favourable.  We begin with the covering
radius.

\begin{theorem}[Covering radius improvement]\label{thm:eps_improvement}
  Let $\varepsilon_U^*$ denote the smallest $\varepsilon > 0$
  satisfying the unfolded covering condition
  \begin{equation}\label{eq:covering_cond_U}
    \Bigl(\frac{D_U}{\varepsilon} + 1\Bigr)^{\!d}
    \cdot
    \bigl(1 - \pi_{\min}\, V_d\, \varepsilon^d\bigr)^{n_{\mathrm{eff}}}
    \;\leq\; \frac{\delta}{3},
  \end{equation}
  and let $\varepsilon_F^*$ be defined
  by~\eqref{eq:covering_cond_F} with $h_{\max} = 1$.  Then
  $\varepsilon_F^* \leq \varepsilon_U^*$.
\end{theorem}

\begin{proof}
  Write the left-hand side of the covering condition as
  $\Phi(D_K,\, \pi_{\min},\, \varepsilon)
  = (D_K/\varepsilon + 1)^d\,
  (1 - \pi_{\min}\, V_d\, \varepsilon^d)^{n_{\mathrm{eff}}}$.
  For any fixed $\varepsilon > 0$:
  \begin{enumerate}[nosep]
    \item $D_F \leq D_U$
    (Corollary~\ref{cor:hpd_geometry}a), so the first factor does not
    increase;
    \item $\pi_{\min}^F \geq s\,\pi_{\min} \geq \pi_{\min}$
    (Corollary~\ref{cor:hpd_geometry}b), so
    $\pi_{\min}^F V_d \varepsilon^d \geq \pi_{\min} V_d \varepsilon^d$,
    which makes the base of the second factor smaller and hence the
    second factor does not increase.
  \end{enumerate}
  Therefore $\Phi(D_F, \pi_{\min}^F, \varepsilon)
  \leq \Phi(D_U, \pi_{\min}, \varepsilon)$ for all
  $\varepsilon > 0$.
  Since $\varepsilon_U^*$ satisfies
  $\Phi(D_U, \pi_{\min}, \varepsilon_U^*) \leq \delta/3$, we have
  $\Phi(D_F, \pi_{\min}^F, \varepsilon_U^*) \leq \delta/3$, so by
  minimality $\varepsilon_F^* \leq \varepsilon_U^*$.
  The same monotonicity argument extends to the stabiliser-corrected
  case whenever $(s/h_{\max})\,\pi_{\min}^D \geq \pi_{\min}$, which
  holds in the finite-group settings considered here.
\end{proof}

\begin{remark}[Asymptotic scaling]\label{rem:eps_scaling}
  Under simplifying approximations (dropping the $+1$ term and using
  $\log(1 - p) \approx -p$), the ratio
  $\varepsilon_F^* / \varepsilon_U^*$ scales heuristically as
  $s^{-1/d}$, up to logarithmic factors.  This explains the
  empirical observation that the covering-radius improvement is
  more pronounced in low dimensions
  (Section~\ref{sec:experiments}).  However, the rigorous content
  of Theorem~\ref{thm:eps_improvement} is the inequality
  $\varepsilon_F^* \leq \varepsilon_U^*$, which is independent of
  the cross-mode misspecification condition.
\end{remark}

\subsection{Diagnostic improvement}
\label{sec:gap_improvement}

The covering-radius improvement (Theorem~\ref{thm:eps_improvement})
guarantees that the radius factor $\varepsilon_F^*$ does not
increase.  The full product $M_F\varepsilon_F^*$ need not be
smaller unless $M_F$ is comparable to~$M_U$, which is why the
sufficient condition below keeps both terms explicit.
The dominant source of improvement, however, is the empirical
oscillation $\widehat{\osc}_n^F$, which drops when the quotient
proposal eliminates cross-mode residuals.

We now compare the folded and unfolded diagnostics.
Let $T_U\colon \R^d \to \R^d$ be a normalising flow trained on the
original (unfolded) target~$\pi$, with proposal density
$q_U = (T_U)_{\#}\cN(0,I)$ and log-density ratio
$h_U = \log\pi - \log q_U$.
The unfolded diagnostic is
$C_U = \widehat{\osc}_n + 2 M_U \varepsilon_U^*$, where
$\widehat{\osc}_n$, $M_U$, $\varepsilon_U^*$ are defined
analogously to their folded counterparts but using $h_U$
and~$\cK_\alpha$.

To quantify the oscillation gap we introduce three conditions.

\begin{description}[leftmargin=1.5em,nosep]
  \item[(M\,$'$)]
  \textbf{Relative cross-mode deficiency.}
  For each off-mode component $\cK_\alpha^{(j)}$, $j \neq 1$,
  there exists $\Delta_j > 0$ such that
  \begin{equation}\label{eq:cond_M}
    \inf_{\theta \in \cK_\alpha^{(j)}} h_U(\theta)
    \;\geq\;
    \sup_{\theta \in \cK_\alpha^{(1)}} h_U(\theta)
    \;+\; \Delta_j.
  \end{equation}

  \item[(W)]
  \textbf{Folded within-mode fit.}
  There exists $r_1 \geq 0$ such that
  \begin{equation}\label{eq:cond_W}
    \osc_{\cK_\alpha^F}\!\bigl(\log(s\pi) - \log q_{T_F}\bigr)
    \;\leq\; 2r_1.
  \end{equation}

  \item[(R)]
  \textbf{Folded-proposal residual.}
  \begin{equation}\label{eq:cond_R}
    r_F \;:=\;
    \sup_{z \in \cK_\alpha^F}
    \bigl|\log q_F(z) - \log q_{T_F}(z)\bigr|
    \;<\; \infty.
  \end{equation}
\end{description}

\noindent
Condition~(M$'$) says the unfolded log-density ratio~$h_U$ is
uniformly higher on every off-mode than on the primary mode: the
unfolded flow under-covers those regions by at least a factor
of~$e^{\Delta_j}$.  This holds whenever the spectrally normalised
RealNVP has insufficient capacity to represent multiple separated
modes---a common situation in practice, as illustrated by the experiments
in Section~\ref{sec:experiments} and explicitly checked in
representative cases in Supplement~\ref{app:condition_check}.
Condition~(W) bounds the within-mode oscillation of the folded
flow's unnormalised log-density ratio on~$\cK_\alpha^F$.
Condition~(R) quantifies the contribution of off-mode orbit
terms $q_{T_F}(g \cdot z)$ for $g \neq e$; for well-separated
modes $r_F \approx 0$ since these terms are exponentially small.

Note that (M$'$) involves the unfolded flow~$T_U$ while (W) and
(R) involve the folded flow~$T_F$; this reflects the two-flow
comparison at the heart of FolT-MCMC.

\begin{lemma}[Oscillation gap]\label{lem:osc_gap}
  Under conditions \textup{(M$'$)}, \textup{(W)}, and
  \textup{(R)}, the empirical oscillations on the unfolded and
  folded spaces satisfy, with probability at least $1 - \delta'$,
  \begin{equation}\label{eq:osc_gap}
    \widehat{\osc}_n
    \;\geq\;
    \widehat{\osc}_n^F
    \;+\; \min_{j \neq 1} \Delta_j
    \;-\; 2r_1 - r_F
    \;-\; c_n,
  \end{equation}
  where $c_n = O\!\bigl(\sqrt{\log(2/\delta')/n}\bigr)$ accounts for
  the sampling error in both empirical oscillations.
\end{lemma}

\begin{proof}
  We bound $\widehat{\osc}_n$ from below and $\widehat{\osc}_n^F$
  from above.

  \emph{Lower bound on $\widehat{\osc}_n$.}\;
  Condition~(M$'$) gives, for any off-mode $j \neq 1$,
  $\inf_{\cK_\alpha^{(j)}} h_U \geq \sup_{\cK_\alpha^{(1)}} h_U
  + \Delta_j$.
  Therefore
  \[
    \osc_{\cK_\alpha}(h_U)
    \;\geq\;
    \inf_{\cK_\alpha^{(j)}} h_U - \inf_{\cK_\alpha^{(1)}} h_U
    \;\geq\;
    \sup_{\cK_\alpha^{(1)}} h_U + \Delta_j - \inf_{\cK_\alpha^{(1)}} h_U
    \;\geq\;
    \min_{j \neq 1}\Delta_j.
  \]
  The empirical oscillation $\widehat{\osc}_n$ concentrates around
  $\osc_{\cK_\alpha}(h_U)$ up to a term $O(\sqrt{\log/n})$.

  \emph{Upper bound on $\widehat{\osc}_n^F$.}\;
  For $z \in \cK_\alpha^F = \cK_\alpha^{(1)} \cap D$ (a.e.),
  \[
    h_F(z) = \log(s\pi(z)) - \log q_F(z)
    = \underbrace{\log(s\pi(z)) - \log q_{T_F}(z)}_{\text{within-mode:
      osc} \leq 2r_1}
    \;-\; \log\!\Bigl(1 + \frac{\sum_{g \neq e} q_{T_F}(gz)}{q_{T_F}(z)}\Bigr).
  \]
  The last term lies in $[-r_F, 0]$ by condition~(R).
  Therefore $\osc_{\cK_\alpha^F}(h_F) \leq 2r_1 + r_F$, and
  $\widehat{\osc}_n^F$ concentrates similarly.

  Combining: $\widehat{\osc}_n - \widehat{\osc}_n^F \geq
  \min_j \Delta_j - 2r_1 - r_F - c_n$.
\end{proof}

We can now state the main comparison result.

\begin{theorem}[Diagnostic improvement]\label{thm:gap_improvement}
  Adopt the setting of Theorem~\ref{thm:folded_cert} and let
  $C_U = \widehat{\osc}_n + 2 M_U \varepsilon_U^*$ be the unfolded
  diagnostic (with gradient constant~$M_U$ and covering
  radius~$\varepsilon_U^*$).
  Suppose conditions \textup{(M$'$)}, \textup{(W)}, \textup{(R)}
  hold, and define
  \[
    \Delta
    \;=\;
    \min_{j \neq 1}\Delta_j - 2r_1 - r_F - c_n.
  \]
  If
  \begin{equation}\label{eq:sufficient}
    \Delta
    \;>\;
    2\bigl(M_F\,\varepsilon_F^*
    \;-\; M_U\,\varepsilon_U^*\bigr)^{\!+}
  \end{equation}
  (where $(x)^+ = \max(x,0)$), then with probability at least
  $1 - \delta_U - \delta_F - \delta'$:
  \begin{equation}\label{eq:C_comparison}
    C_F \;<\; C_U,
    \qquad
    \frac{\underline{\gamma}_{F,\alpha}}{\underline{\gamma}_{U,\alpha}}
    \;=\;
    \frac{1 + \exp(C_U)}{1 + \exp(C_F)}
    \;>\; 1,
  \end{equation}
  where $\underline{\gamma}_{F,\alpha}$ and
  $\underline{\gamma}_{U,\alpha}$ are the respective HPD-core
  convergence diagnostics.
\end{theorem}

\begin{proof}
  Write
  \begin{align*}
    C_U - C_F
    &= \bigl(\widehat{\osc}_n - \widehat{\osc}_n^F\bigr)
       + 2\bigl(M_U\varepsilon_U^* - M_F\varepsilon_F^*\bigr) \\
    &\geq \Delta + 2\bigl(M_U\varepsilon_U^* - M_F\varepsilon_F^*\bigr)
      \qquad\text{(Lemma~\ref{lem:osc_gap})}.
  \end{align*}
  If $M_F\varepsilon_F^* \leq M_U\varepsilon_U^*$, both terms are
  non-negative and $C_U - C_F \geq \Delta > 0$.
  If $M_F\varepsilon_F^* > M_U\varepsilon_U^*$,
  condition~\eqref{eq:sufficient} gives
  $\Delta > 2(M_F\varepsilon_F^* - M_U\varepsilon_U^*)$, so again
  $C_U - C_F > 0$.  The convergence-diagnostic comparison follows
  from the monotonicity of $C \mapsto 2/(1 + e^C)$.
  The confidence level combines the three events via a union bound:
  the folded diagnostic (prob $\geq 1 - \delta_F$), the unfolded
  diagnostic ($\geq 1 - \delta_U$), and the oscillation-gap
  concentration ($\geq 1 - \delta'$).
\end{proof}

\subsection{When does folding help?}
\label{sec:when_folding_helps}

Theorems~\ref{thm:eps_improvement} and~\ref{thm:gap_improvement}
identify two complementary improvement mechanisms: a covering-radius
reduction that holds independently of the cross-mode deficiency
condition ($\varepsilon_F^* \leq \varepsilon_U^*$, under
Assumption~\ref{asm:generic_stab}) and an oscillation
reduction that holds under condition~(M$'$).

\paragraph{The role of (M$'$).}
Condition~(M$'$) is \emph{not} an assumption about the target; it is
a condition on the \emph{unfolded flow}~$T_U$.  It asserts that the
trained unfolded transport map fails to cover at least one off-mode.  This holds
naturally for spectrally normalised RealNVP with scale clipping,
whose limited expressiveness makes faithful representation of
multiple separated modes difficult---especially in higher dimensions,
where the coupling architecture must mediate inter-coordinate
information through the conditioner network.  In representative
experiments (Supplement~\ref{app:condition_check}), condition~(M$'$)
is strongly satisfied, with $\Delta/C_U$ between $0.87$ and $0.90$.

When the flow \emph{is} expressive enough to represent all modes
faithfully, $\Delta \approx 0$ and the oscillation gap vanishes;
the improvement then comes solely from the covering-radius
reduction (Theorem~\ref{thm:eps_improvement}), which is
modest---heuristically $O(s^{-1/d})$
(Remark~\ref{rem:eps_scaling}).

\paragraph{Fold boundary placement.}
The main-text results
(Theorems~\ref{thm:folded_cert}--\ref{thm:gap_improvement})
are stated under Assumption~\ref{asm:generic_stab}, i.e.\ the HPD
set does not touch the fold boundary.  The boundary-touching case
is handled by the stabiliser-corrected covering
condition~\eqref{eq:covering_cond_F} and is detailed in
Supplement~\ref{app:boundary}.
When the fold boundary passes
through a high-density region---as in our banana diagnostic
experiment (Section~\ref{sec:exp_banana})---the full oscillation
can increase due to boundary residuals, even while the
quantile-core diagnostic improves.  This provides a practical
design principle: \emph{the fold boundary should lie in a
low-density region separating the symmetric modes}.
For label-switching mixtures, the sorted-chamber boundary
is automatically in the valley between modes whenever the
components are well separated.

\paragraph{Dimension dependence.}
A key empirical finding is that
QC~$\underline{\gamma}_F$ is empirically nearly dimension-free:
it stays in $[0.90, 0.94]$ across $d = 2$ to $20$ for Gaussian
mixtures, while QC~$\underline{\gamma}_U$ collapses from $0.40$
to $0.016$ (Section~\ref{sec:exp_dimension}).  The theoretical
explanation is that all three terms in $C_F$ remain small as $d$
grows: $\widehat{\osc}_n^F$ reflects only the single-mode
transport residual, $M_F$ is controlled by spectral normalisation,
and $\varepsilon_F^*$ benefits from the $s$-fold density floor.
In contrast, $C_U$ is dominated by the cross-mode oscillation,
which grows with~$d$ because the coupling-layer architecture
struggles increasingly to generate multimodal structure in high
dimensions.

\paragraph{Mode count vs.\ dimension.}
The label-switching experiments (Section~\ref{sec:exp_labelswitching})
reveal that the unfolded diagnostic degrades primarily with the
number of modes~$s = m!$, not with the dimension~$d$.
Comparing $m = 3, p = 2$ ($d = 6$, $s = 6$) with $m = 3, p = 4$
($d = 12$, $s = 6$): doubling the dimension while keeping $s$
fixed reduces QC~$\underline{\gamma}_U$ by only a factor of~$2$,
whereas increasing $s$ from $6$ to $24$ (at fixed block size)
reduces it by a factor of~$15$.  The folded diagnostic varies
much less across both changes.

%% file: section5_experiments.tex

\section{Experiments}
\label{sec:experiments}

We evaluate FolT-MCMC on three synthetic target families designed to
isolate the mechanisms predicted by the theory, a comparison with
the random permutation sampler \citep{FruhwirthSchnatter2001},
and a standard Bayesian Gaussian mixture posterior.  The first
experiment is a diagnostic stress test showing that folding can be
harmful when the fold boundary intersects high-density mass.  The
second evaluates dimension scaling for a well-separated
reflection-symmetric Gaussian mixture.  The third evaluates
permutation folding for label-switching mixtures.  The fourth
compares FolT-MCMC with the standard label-switching baseline.  The
fifth applies all methods to a textbook Bayesian three-component
mixture posterior.  Across experiments, we report the quantile-core
convergence diagnostic, NLL, acceptance, and effective
sample size (ESS).
All experiments use the same architecture and training protocol;
the only variable across conditions is whether folding is applied.

\subsection{Setup}
\label{sec:exp_setup}

\paragraph{Architecture.}
Both the folded and unfolded flows are spectrally normalised
RealNVP \citep{DinhSohlDickstein2017,Miyato2018} with
$\tanh$~activations and scale clipping at $c = 0.7$, following
\citet{LCNF}.  Layer count and hidden dimension
scale with~$d$: $(L, h) = (8, 64)$ for $d \leq 4$;
$(10, 128)$ for $d = 5$--$6$; $(12, 128)$ for $d = 8$--$12$;
$(16, 256)$ for $d = 20$.  The folded and unfolded pipelines share
identical architecture for each~$d$.

\paragraph{Training.}
Negative log-likelihood with annealed oscillation and gradient
regularisation (FolT-OG-Anneal), as in \citet{CerTMCMC}.
Training uses Adam with learning rate $10^{-3}$, batch size 512,
and 2\,000 epochs for $d \leq 6$ (increasing to 3\,500 for
$d \geq 8$).  Training samples are drawn from the exact sampler
of each target (folded samples projected onto~$D$).

\paragraph{Diagnostic computation.}
The quantile-core convergence diagnostic (denoted QC~$\gamma$)
at $\rho = 0.05$ (primary) and
$\rho = 0.025$ (secondary), with $n = 20\,000$ independent
diagnostic samples and Dvoretzky--Kiefer--Wolfowitz
(DKW)-corrected confidence $\delta = 0.05$,
following the protocol of \citet{CerTMCMC}.

\paragraph{MCMC.}
Independence Metropolis--Hastings with 4 chains of length 5\,000
(burn-in 1\,000).  Effective sample size estimated via
batch means.

\paragraph{Hardware and code.}
NVIDIA RTX~4090 GPU; code will be released upon acceptance.

\subsection{Diagnostic: fold boundary location}
\label{sec:exp_banana}

\paragraph{Purpose.}
To demonstrate that the fold boundary's position relative to
the target's density mass is a critical design variable,
and that folding can \emph{hurt} the full oscillation bound
when the boundary passes through a high-density region.

\paragraph{Target.}
Asymmetric double banana in $d = 2$: a mixture of two
banana-shaped modes with curvatures $a = 1.0$ and $b = 0.5$,
related by reflection $\theta_1 \mapsto -\theta_1$.
Because $a \neq b$, the target is not strictly
$\mathbb{Z}_2$-invariant; this experiment intentionally
stresses the method outside the exact group-invariant setting
of Section~\ref{sec:framework}.
The fold boundary $\{\theta_1 = 0\}$ passes through the
high-density region near the origin where both modes overlap.

\paragraph{Results.}
Table~\ref{tab:banana} summarises the comparison.

\begin{table}[t]
\centering
\caption{Banana diagnostic ($d = 2$, $s = 2$).
  The fold boundary passes through a high-density region.}
\label{tab:banana}
\smallskip
\begin{tabular}{lcc}
\hline
Metric & Unfolded & Folded \\
\hline
Full oscillation & $30.7$ & $58.4$ \\
QC oscillation ($\rho = 0.05$) & $1.38$ & $1.10$ \\
QC $\underline{\gamma}$ ($\rho = 0.05$) & $0.272$ & $0.320$ \\
QC $\underline{\gamma}$ ($\rho = 0.025$) & $0.064$ & $0.135$ \\
Final NLL & $4.34$ & $3.34$ \\
ESS\,/\,sample & $0.048$ & $0.042$ \\
Acceptance rate & $0.604$ & $0.594$ \\
\hline
\end{tabular}
\end{table}

The full oscillation \emph{increases} under folding ($58.4$ vs.\
$30.7$), because the smooth flow cannot produce the hard truncation
at $\theta_1 = 0$ where the target density is high, creating a
boundary residual spike.  Nevertheless, the quantile-core
convergence diagnostic still improves: QC~$\underline{\gamma}$ rises
from $0.272$ to $0.320$ at $\rho = 0.05$ and from $0.064$ to $0.135$
at $\rho = 0.025$, because the $\rho$-trimming excludes the boundary
spike.  The NLL drops by $23\%$ ($4.34 \to 3.34$), confirming
that the flow fits the single folded mode more easily.

This experiment validates the design principle of
Section~\ref{sec:when_folding_helps}: \emph{folding helps when the
fold boundary lies in a low-density valley, not when it cuts through
the target's mass}.  The subsequent experiments use targets
satisfying this condition.

\subsection{Dimension scaling: well-separated mixture}
\label{sec:exp_dimension}

\paragraph{Purpose.}
To test whether the folded diagnostic is robust to increasing
dimension, in contrast to the unfolded diagnostic.

\paragraph{Target.}
Symmetric two-component Gaussian mixture,
$\pi = \tfrac{1}{2}\cN(\mu_1, I_d) + \tfrac{1}{2}\cN(\mu_2, I_d)$,
with $\mu_1 = (3, 0, \ldots, 0)$, $\mu_2 = -\mu_1$
(separation $\delta = 6$).
Reflection fold with $G = \mathbb{Z}_2$, $s = 2$.
Dimensions $d \in \{2, 5, 10, 20\}$.  The fold boundary
$\{\theta_1 = 0\}$ sits in the inter-mode valley where the density
relative to a single-component peak is
$\exp(-9/2) \approx 0.011$.

\paragraph{Results.}
Table~\ref{tab:mixture_scaling} and Figure~\ref{fig:headline} present the key comparison.

\begin{table}[t]
\centering
\caption{Dimension scaling for Gaussian mixture
  ($s = 2$, $\rho = 0.05$).}
\label{tab:mixture_scaling}
\smallskip
\begin{tabular}{rccccc}
\hline
$d$ & QC $\underline{\gamma}_U$ & QC $\underline{\gamma}_F$
    & Ratio
    & QC osc$_U$ & QC osc$_F$ \\
\hline
2  & $0.402$ & $0.936$ & $2.3\times$  & $1.38$ & $0.13$ \\
5  & $0.094$ & $0.941$ & $10.0\times$ & $3.01$ & $0.12$ \\
10 & $0.016$ & $0.922$ & $59.2\times$ & $4.85$ & $0.16$ \\
20 & $0.016$ & $0.902$ & $57.0\times$ & $4.83$ & $0.20$ \\
\hline
\end{tabular}
\end{table}

\begin{figure}[t]
\centering
\includegraphics[width=0.65\textwidth]{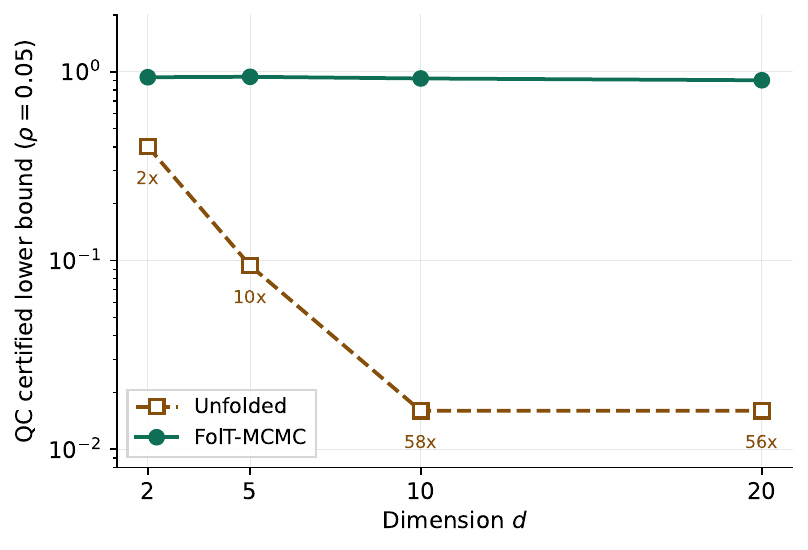}
\caption{Quantile-core convergence diagnostic vs.\ dimension
  (Gaussian mixture, $s = 2$, $\rho = 0.05$).
  The folded diagnostic is empirically nearly dimension-free;
  the unfolded diagnostic collapses.
  Numbers indicate the improvement ratio at each dimension.}
\label{fig:headline}
\end{figure}

The folded quantile-core diagnostic QC~$\underline{\gamma}_F$
remains in $[0.90, 0.94]$ across $d = 2$ to $20$, while the
unfolded diagnostic collapses from $0.40$ to $0.016$---a
$59$-fold gap by $d = 10$.  The quantile-core oscillation bound
tells the same story: folded QC oscillation stays below $0.20$ at
all dimensions, while the unfolded QC oscillation grows from
$1.38$ to $4.85$.

At $d = 20$, the folded NLL of $28.48$ is within $0.1$ of the
entropy of a $d$-dimensional standard Gaussian,
$\tfrac{d}{2}\log(2\pi e) \approx 28.4$, indicating that the
folded transport achieves a near-optimal fit.
The unfolded flow, by contrast, must represent a bimodal target in
$20$ dimensions using a coupling architecture that can only
influence each coordinate conditionally---a task at which it
demonstrably fails.

\subsection{Label-switching scaling}
\label{sec:exp_labelswitching}

\paragraph{Purpose.}
To verify that permutation folding gives the same strong
improvement as reflection folding, and to disentangle the effects
of mode count~$s$ and dimension~$d$ on the unfolded diagnostic.

\paragraph{Target.}
Label-switching Gaussian mixtures with $m$~components, each a
block of $p$~parameters (representing, e.g., frequency and damping
of a structural mode), total dimension $d = mp$.  Component centres
are spaced by $3$ standard deviations along the sort coordinate.
Permutation fold with $G = S_m$, $s = m!$.  Four configurations
are tested (Table~\ref{tab:labelswitching}).

\paragraph{Results.}

\begin{table}[t]
\centering
\caption{Label-switching scaling (quantile-core, $\rho = 0.05$).}
\label{tab:labelswitching}
\smallskip
\begin{tabular}{lccccccc}
\hline
Config & $d$ & $s$ & QC $\underline{\gamma}_U$ & QC $\underline{\gamma}_F$
       & Ratio & QC osc$_U$ & QC osc$_F$ \\
\hline
$m{=}2,\,p{=}2$ & $4$  & $2$   & $0.366$ & $0.820$ & $2.2\times$   & $1.50$ & $0.37$ \\
$m{=}3,\,p{=}2$ & $6$  & $6$   & $0.066$ & $0.685$ & $10.4\times$  & $3.38$ & $0.65$ \\
$m{=}3,\,p{=}4$ & $12$ & $6$   & $0.036$ & $0.683$ & $19.1\times$  & $4.01$ & $0.66$ \\
$m{=}4,\,p{=}2$ & $8$  & $24$  & $0.0043$ & $0.624$ & $145\times$   & $6.14$ & $0.79$ \\
\hline
\end{tabular}
\end{table}

\begin{figure}[t]
\centering
\includegraphics[width=0.65\textwidth]{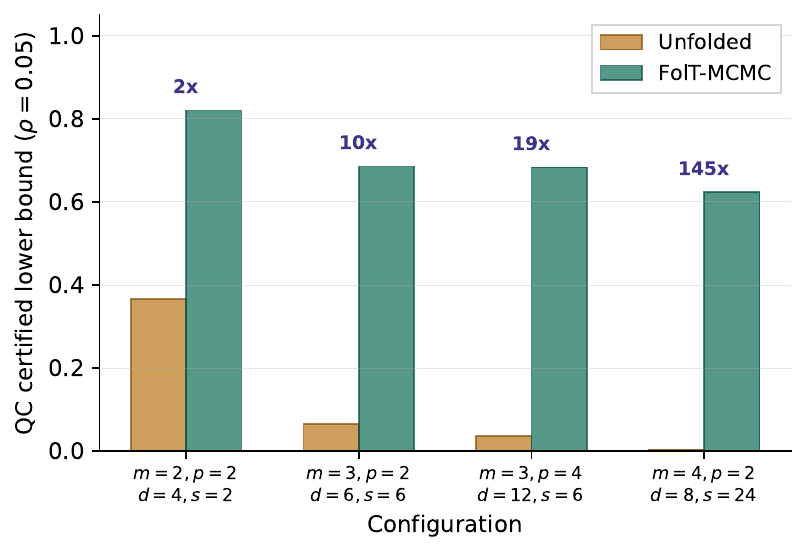}
\caption{Quantile-core convergence diagnostic across label-switching
  configurations.  The folded diagnostic varies slowly;
  the unfolded diagnostic collapses superlinearly with~$s$.
  Numbers indicate the improvement ratio.}
\label{fig:labelswitching}
\end{figure}

The folded quantile-core diagnostic
QC~$\underline{\gamma}_F$ stays in $[0.62, 0.82]$
across all configurations, while the unfolded diagnostic collapses
from $0.366$ to $0.0043$ as the number of equivalent modes grows
from $s = 2$ to $s = 24$.  The improvement ratio reaches
$145\times$ at $m = 4$.

\paragraph{Disentangling mode count from dimension.}
Comparing $m = 3, p = 2$ ($d = 6$, $s = 6$,
QC~$\underline{\gamma}_U = 0.066$) with $m = 3, p = 4$ ($d = 12$,
$s = 6$, QC~$\underline{\gamma}_U = 0.036$): doubling the dimension
at fixed~$s$ reduces QC~$\underline{\gamma}_U$ by a factor of
${\approx}\,2$.  Comparing $m = 3, p = 2$ ($s = 6$) with
$m = 4, p = 2$ ($s = 24$): increasing $s$ by $4\times$ at
comparable~$d$ reduces QC~$\underline{\gamma}_U$ by a factor of
${\approx}\,15$.  In this controlled comparison, the unfolded
diagnostic appears more sensitive to the number of symmetric
copies than to the ambient dimension---consistent with the
theoretical prediction of
Section~\ref{sec:when_folding_helps}.  The folded diagnostic
varies much less across both changes
(QC~$\underline{\gamma}_F \in [0.68, 0.68]$ for the two $s = 6$
configurations despite $d$ doubling).

\paragraph{Convergence at $m = 4$.}
With $s = 24$ equivalent modes, the unfolded flow trains stably
(no divergence or numerical issues) but fails representationally:
NLL $= 15.0$, acceptance rate $= 0.41$, full-oscillation
diagnostic $\gamma_{\mathrm{full}} \approx 5 \times 10^{-10}$
(completely vacuous).  Scatter plots confirm that the unfolded
proposal spreads mass across the $24$ symmetric chambers but yields
a highly uneven density ratio, leading to low acceptance and a
vacuous diagnostic.  The folded chain, by contrast, samples
efficiently in the sorted chamber ($\mathrm{acceptance} = 0.90$,
$\mathrm{ESS/sample} = 0.85$).


\subsection{Comparison with random permutation sampler}
\label{sec:exp_baseline}

\paragraph{Purpose.}
To compare FolT-MCMC with the random permutation sampler
(RPS) of \citet{FruhwirthSchnatter2001}, the most widely used
label-switching mitigation in Bayesian mixture modelling.

\paragraph{Setup.}
All three methods use the $m = 3$, $p = 2$ label-switching target
($d = 6$, $s = 6$) from Section~\ref{sec:exp_labelswitching}.
The unfolded IMH and RPS share the same trained flow as proposal;
RPS additionally applies a random component-label permutation
after each MH accept/reject step.
FolT-MCMC uses its own folded flow in the sorted chamber.
Chains: $4 \times 10\,000$ steps, burn-in $2\,000$.
To ensure a fair mixing comparison, we report ESS after projecting
all samples into the sorted fundamental domain.

\paragraph{Results.}

\begin{table}[t]
\centering
\caption{Baseline comparison ($m = 3$, $d = 6$, $s = 6$).
  Sorted ESS is computed after projecting all samples into the
  sorted chamber.}
\label{tab:baseline}
\smallskip
\begin{tabular}{lccc}
\hline
Method & Sorted ESS\,/\,sample & Acceptance
       & QC $\underline{\gamma}$ \\
\hline
Unfolded IMH & $0.607$ & $0.622$ & $0.066$ \\
FolT-MCMC    & $0.540$ & $0.881$ & $0.685$ \\
RPS          & $0.482$ & $0.612$ & N/A \\
\hline
\end{tabular}
\end{table}

The sorted ESS per sample is comparable across all three methods
($0.48$--$0.61$), well within the variance of the batch-means
estimator.  This is expected: all three use a global independence
proposal from a well-trained flow that already covers all six
modes, so no chain gets stuck in a single mode---the setting where
RPS was designed to help \citep{FruhwirthSchnatter2001}.
(The per-sample permutation step in our RPS implementation is
not batched, but wall-clock time is not the basis of this
comparison.)

The key distinction is structural rather than computational.
FolT-MCMC's folded diagnostic QC~$\underline{\gamma} = 0.685$ predicts
$\mathrm{ESS}/n \approx \gamma/(2-\gamma) = 0.52$, matching the
empirical $0.54$---a tight diagnostic.
The unfolded diagnostic QC~$\underline{\gamma} = 0.066$ predicts
$\mathrm{ESS}/n \approx 0.034$, far below the empirical $0.61$:
the diagnostic is valid but loose, reflecting worst-case
cross-mode oscillation that did not materialise on this particular
run.
The random permutation sampler can improve label exploration,
but its composite IMH-plus-permutation kernel does not fit the
independence-proposal minorisation structure that the convergence
diagnostic requires.  FolT-MCMC, by contrast, admits a non-vacuous
diagnostic because it samples from a single-mode target with a
learned proposal.

This comparison underscores the core thesis of FolT-MCMC:
\emph{folding does not primarily improve raw mixing speed; it
restores a non-vacuous convergence diagnostic for the chain},
closing the gap between the provable lower bound and the empirical
performance.

\subsection{Standard Bayesian mixture posterior}
\label{sec:exp_bayesian_mixture}

\paragraph{Purpose.}
To validate FolT-MCMC on a textbook Bayesian inference problem
rather than a method-specific synthetic target, and to demonstrate
that a non-vacuous convergence diagnostic requires quotient-space
sampling even when post-hoc relabelling recovers correct point
estimates.  This separates point-estimation recovery from
a reliable convergence diagnostic.

\paragraph{Model.}
$N = 500$ observations drawn from a three-component univariate
Gaussian mixture with equal weights $w = (1/3, 1/3, 1/3)$ and
unknown component means and standard deviations.
Parameters: $\theta = (\mu_1, \log\sigma_1, \mu_2, \log\sigma_2,
\mu_3, \log\sigma_3)$, $d = 6$.
Priors: $\mu_k \sim \cN(3, 5^2)$,
$\log\sigma_k \sim \cN(0, 0.5^2)$.
The equal weights and exchangeable priors give the posterior
exact $S_3$ label-switching symmetry ($s = 6$).
Training samples are generated by a random-walk Metropolis
pre-run of $200\,000$~steps, thinned to $50\,000$.

\FloatBarrier

\paragraph{Methods compared.}
Four pipelines, all using the same RealNVP architecture
(10~layers, hidden dim 128):
(i)~unfolded IMH; (ii)~FolT-MCMC with permutation fold
($D = \{\mu_1 \leq \mu_2 \leq \mu_3\}$);
(iii)~RPS (unfolded flow $+$ random permutation);
(iv)~post-hoc sort (unfolded chain samples sorted by~$\mu_k$).

\FloatBarrier

\paragraph{Results.}

\begin{table}[t]
\centering
\caption{Bayesian mixture benchmark
  ($K{=}3$, $d{=}6$, $s{=}6$, $N{=}500$,
  $w = (1/3, 1/3, 1/3)$).
  True parameters: $\mu = (0, 3, 6)$,
  $\sigma = (0.8, 1.0, 0.7)$.
  Estimates $\hat\mu_k$ are sorted posterior means for all methods
  except unfolded IMH, which reports raw labelled means from the
  single mode in which the chain is stuck.}
\label{tab:bayesian_mixture}
\smallskip
\begin{tabular}{lccccc}
\hline
Method & QC $\underline{\gamma}$ & Accept
       & $\hat\mu_1$ & $\hat\mu_2$ & $\hat\mu_3$ \\
\hline
Unfolded IMH & ${\approx}0$ & ${\approx}0$
             & $1.15$ & $5.60$ & $2.07$ \\
FolT-MCMC    & $0.878$ & $0.960$
             & $-0.12$ & $2.95$ & $5.98$ \\
RPS          & N/A & $0.001$
             & $-0.11$ & $3.02$ & $5.99$ \\
Post-hoc sort & ${\approx}0^{\dagger}$ & ${\approx}0$
              & $-0.16$ & $2.99$ & $5.99$ \\
\hline
\end{tabular}

\smallskip
{\footnotesize ${}^{\dagger}$\,Inherited from the unfolded chain; post-hoc sorting does not define an independent sampler.}
\end{table}

\begin{figure}[t]
\centering
\includegraphics[width=\textwidth]{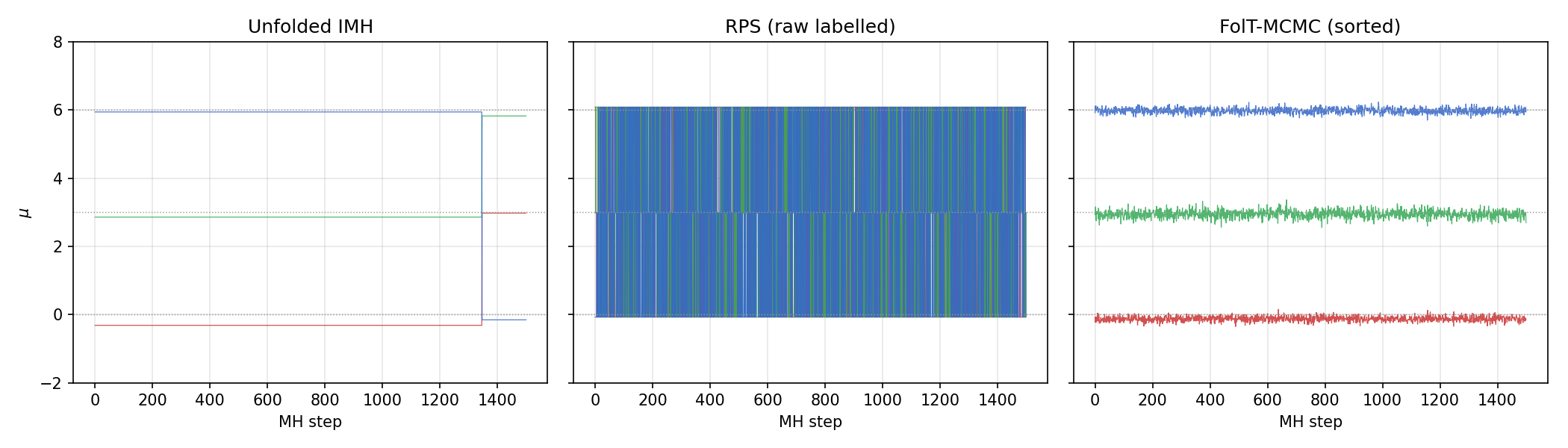}
\caption{Trace plots for the Bayesian mixture posterior.
  Left: unfolded IMH is frozen (acceptance ${\approx}0$).
  Centre: RPS exhibits textbook label switching in raw
  labelled coordinates.
  Right: FolT-MCMC in sorted space shows clean separated bands.}
\label{fig:bayesian_trace}
\end{figure}

The unfolded RealNVP cannot fit the six equally weighted,
well-separated posterior modes (full oscillation $= 443$), so the
unfolded IMH chain freezes entirely: acceptance rate~${\approx}0$,
and the raw labelled parameter estimates
($\hat\mu = 1.15, 5.60, 2.07$, not sorted) reflect the single
mode in which the chain is stuck rather than the true posterior.
The diagnostic is completely vacuous
(QC~$\underline{\gamma} \approx 10^{-27}$).

FolT-MCMC, sampling in the sorted chamber, achieves acceptance
$0.960$ and recovers the true parameters accurately:
$\hat\mu = (-0.12, 2.95, 5.98)$,
$\hat\sigma = (0.78, 0.92, 0.81)$.  The quantile-core diagnostic
QC~$\underline{\gamma} = 0.878$ is the only non-vacuous diagnostic
among the four methods.

\paragraph{Post-hoc relabelling: correct point estimates,
no guarantee.}
Under exact $S_3$ symmetry with equal weights, all six posterior
modes have identical shape.  Consequently, post-hoc sorting of a
frozen chain yields correct sorted point estimates
($\hat\mu = -0.16, 2.99, 5.99$), because sorting any single
mode's samples recovers the true ordered parameters regardless of
which mode the chain occupies.  However, the chain has not
explored the posterior: it provides no reliable uncertainty
quantification and only a vacuous convergence diagnostic
(QC~$\underline{\gamma} \approx 0$).
This reveals a subtle but important distinction:
\emph{post-hoc relabelling can recover correct point estimates
under exact symmetry, but it cannot verify that the sampler has
converged or that credible intervals are valid}.
Quotient-space sampling is needed not for point estimation per se,
but for diagnosing posterior convergence.

\paragraph{RPS: correct estimates, no diagnostic.}
The random permutation sampler, using the same frozen flow but
forcing label permutations, achieves textbook label switching
(visible in the raw trace, Figure~\ref{fig:bayesian_trace} centre)
and recovers accurate sorted estimates
($\hat\mu = -0.11, 3.02, 5.99$).
However, its MH acceptance rate remains $0.001$ (the permutation
step does not improve proposal quality), and it provides no
non-vacuous convergence diagnostic.

%% file: section6_application.tex

\section{Application: closely-spaced structural modes}
\label{sec:application}

We apply FolT-MCMC to Bayesian modal identification of a
supertall building instrumented during Typhoon Mangkhut (2018).
Three closely-spaced modes in the $0.7$--$1.1$\,Hz band create
an $S_3$ label-switching problem ($s = 6$) that renders the
unfolded diagnostic vacuous.

\subsection{Data and model}
\label{sec:app_setup}

\paragraph{Data.}
Accelerometer recordings from 15 usable channels at 32\,Hz.
We select a single 30-minute window near peak typhoon intensity
(band signal-to-noise ratio ${\approx}\,18$\,dB) and compute
the trace of the $15 \times 15$ cross-power spectral density
matrix over $0.7$--$1.1$\,Hz (51~frequency ordinates),
capturing the total vibrational power across all channels.
Building identity and location are anonymised per the owner's
request.

\paragraph{Model.}
A simplified Whittle likelihood with three single-degree-of-freedom
Lorentzian peaks plus a flat noise floor.  Each mode~$j$ has two
parameters (natural frequency~$f_j$, damping ratio~$\xi_j$),
giving $d = 6$ total.  The overall amplitude is profiled out in
closed form (conditional MLE), preserving exact $S_3$~symmetry.
Uniform priors: $f_j \in [0.7, 1.1]$\,Hz,
$\xi_j \in [0.001, 0.05]$.

This is deliberately simpler than a full multi-output Bayesian
operational modal analysis: the goal is to exhibit the
label-switching phenomenon and demonstrate FolT-MCMC's
convergence-diagnostic advantage, not to compete on estimation accuracy.

\paragraph{Folding.}
Permutation fold with $G = S_3$, $s = 6$.  The fundamental domain
is the sorted chamber $D = \{f_1 \leq f_2 \leq f_3\}$.
Parameters are block-whitened (one shared mean and standard
deviation for the frequency slots, one for the damping slots) to
bring all coordinates to unit scale while preserving the
commutativity of whitening and sorting.

\subsection{Results}
\label{sec:app_results}

Training uses 50\,000 pre-samples from a long-run random-walk
Metropolis chain, folded onto~$D$.  Architecture and diagnostic
protocol follow Section~\ref{sec:exp_setup}.

\begin{table}[t]
\centering
\caption{Structural modal identification
  ($d = 6$, $s = 6$, Typhoon Mangkhut window).}
\label{tab:structural}
\smallskip
\begin{tabular}{lcccc}
\hline
Method & QC $\underline{\gamma}$ ($\rho{=}0.05$)
       & Acceptance & ESS\,/\,sample & Label-switch rate \\
\hline
Unfolded & $9.9 \times 10^{-6}$ (vacuous)
         & $0.204$ & $0.059$ & $0.170$ \\
FolT-MCMC & $0.0018$
          & $0.371$ & $0.080$ & $0.000$ \\
\hline
\end{tabular}
\end{table}

\begin{figure}[t]
\centering
\includegraphics[width=\textwidth]{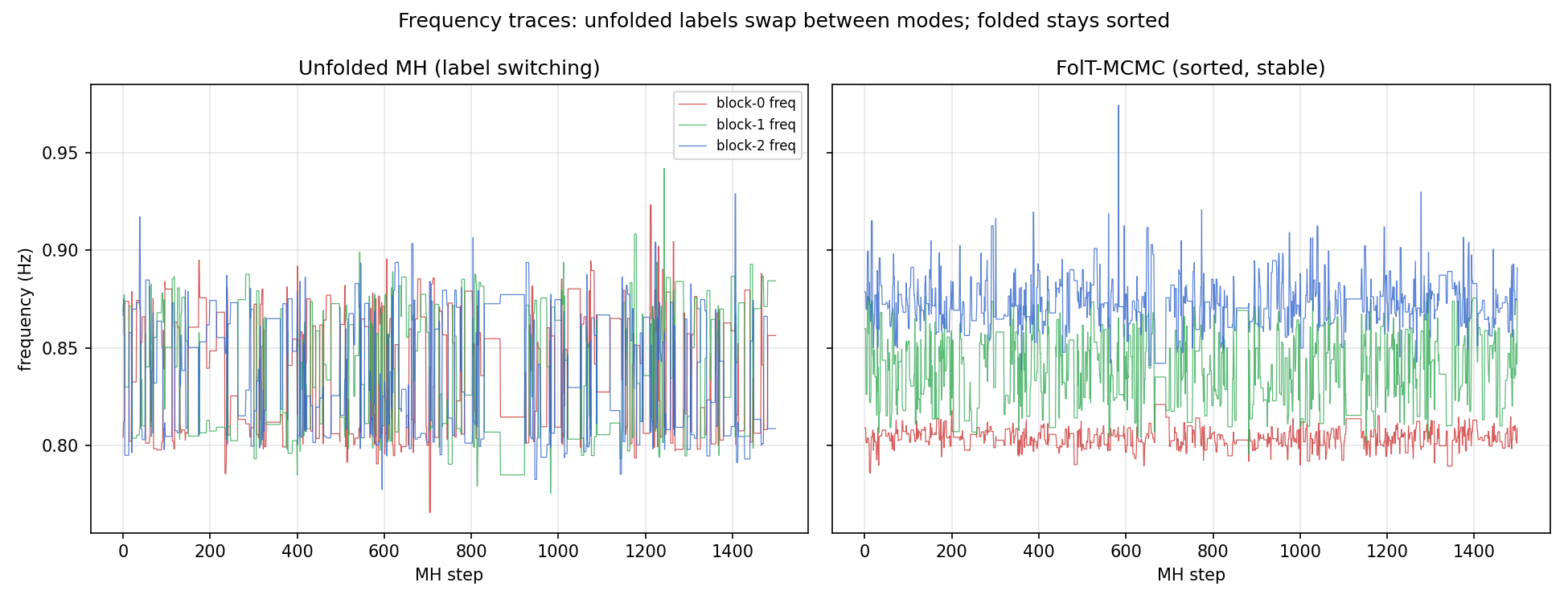}
\caption{Trace plots of the three identified frequencies.
  Left: unfolded chain with frequent label switching
  (17\% of steps).
  Right: FolT-MCMC in sorted space with no label switching.}
\label{fig:structural_trace}
\end{figure}

\paragraph{Label switching.}
The unfolded chain exhibits dense label switching, with $17\%$
of steps swapping the frequency assignment among the three modes
(Table~\ref{tab:structural}).
The FolT-MCMC chain in sorted space shows zero label switches:
the sorted frequencies remain cleanly separated throughout the run.

\paragraph{Convergence diagnostic.}
The unfolded quantile-core diagnostic is
QC~$\underline{\gamma} \approx 10^{-5}$---completely vacuous, as
the six permutation modes inflate the oscillation bound beyond
any useful range.  FolT-MCMC provides a non-vacuous quantile-core diagnostic
(QC~$\underline{\gamma} = 0.0018$), a $180\times$
improvement.  The absolute value is modest, reflecting the genuine
difficulty of closely-spaced modes: the near-degenerate boundary
$f_i \approx f_j$ in the sorted chamber creates heavy-tailed
density ratios that resist a tight diagnostic.  Nevertheless, the
qualitative transition from \emph{no usable diagnostic} to
\emph{a usable diagnostic} is the operationally meaningful distinction.

\paragraph{Posterior estimates.}
The posterior median frequencies are
$\hat{f}_1 = 0.804$, $\hat{f}_2 = 0.838$, $\hat{f}_3 = 0.873$\,Hz,
recovering the closely-spaced cluster structure.  These values are
approximately $0.05$\,Hz below the nominal multi-output estimates
from a separate full Bayesian OMA analysis of the same dataset
\citep{SATMCMCpaper}, a systematic offset attributable to the
single-output trace-spectrum model, which under-weights the
transverse mode (${\approx}\,0.93$\,Hz) relative to the
translational modes.  The demonstration targets the sampler's
convergence diagnostic, not estimation accuracy;
the full AD-BOMA model is developed separately
\citep{SATMCMCpaper}.

%% file: section7_discussion.tex

\section{Discussion}
\label{sec:discussion}

\subsection{Design principles}
\label{sec:disc_design}

The experiments point to a simple but consequential design rule:
\emph{the fold boundary should lie in a low-density valley between
the symmetric modes}.  When this condition holds (Gaussian mixture,
label-switching targets), folding yields large diagnostic gains
across all metrics.  When it is violated (banana diagnostic with
fold boundary through the density ridge), the full oscillation
bound can increase, though the quantile-core diagnostic still
improves modestly.

For well-separated label-switching mixtures, the sorted-chamber
boundary $\{\theta_1^{(j)} = \theta_1^{(j+1)}\}$ lies in a
low-density valley between adjacent mode centres.
For closely-spaced structural modes, however, the boundary may
carry non-negligible posterior mass; this explains why the
structural example in Section~\ref{sec:application} yields a
non-vacuous but modest diagnostic
(QC~$\underline{\gamma} = 0.0018$).

\subsection{Relation to existing work}
\label{sec:disc_related}

\paragraph{Equivariant normalising flows.}
\citet{Kohler2020} and \citet{Rezende2019} construct flows
satisfying $T(g \cdot z) = g \cdot T(z)$, encoding the symmetry
into the architecture.  This reduces the number of learnable
parameters and ensures that the flow respects the group structure,
but does not reduce the number of modes: the equivariant flow must
still represent all $|G|$ copies.  FolT-MCMC takes the dual
approach of collapsing the $|G|$ copies into one via the quotient
map.  The two strategies are complementary: one could train an
equivariant flow and then fold, combining parameter efficiency
with mode reduction.

\paragraph{Post-hoc relabelling.}
Methods such as pivotal reordering \citep{Stephens2000},
equivalence-class representatives \citep{Papastamoulis2010}, and
random permutation samplers \citep{FruhwirthSchnatter2001}
operate after MCMC has been run.  They improve the
interpretability of the samples but do not address the mixing
difficulty of the underlying chain.  In particular, standard post-hoc relabelling methods do not
provide the oscillation-based convergence diagnostic considered here.  FolT-MCMC
resolves the mixing problem at source by eliminating the
symmetric modes before sampling, and inherits its convergence
diagnostic from the framework of \citet{LCNF}.

\paragraph{Tempering and replica exchange.}
Parallel tempering \citep{Geyer1991,Earl2005} overcomes
multimodality by introducing heated replicas that can cross
energy barriers.  This is a general-purpose strategy that does
not exploit symmetry.  FolT-MCMC exploits the \emph{structure} of
the multimodality (it arises from a known group action) to remove
it geometrically, without the overhead of maintaining multiple
temperature levels.

\subsection{Limitations and future work}
\label{sec:disc_limitations}

\paragraph{Exact symmetry.}
The current framework requires the target to be exactly
$G$-invariant.  The simplified exchangeable model in
Section~\ref{sec:application} preserves exact label symmetry by
construction (identical Lorentzian parameterisation and profiled
amplitude).  In more realistic Bayesian operational modal analysis
models, however, mode-specific priors, spatial mode-shape
information, or physical ordering constraints may break this
symmetry approximately.  Extending FolT-MCMC to
$\varepsilon$-approximate group invariance, where
$\|\pi(g\theta) - \pi(\theta)\|$ is small but nonzero, is a
natural direction.  We expect the quotient diagnostic to degrade
gracefully with the symmetry-breaking magnitude, but a formal
perturbation bound remains to be established.

\paragraph{Unknown symmetry groups.}
FolT-MCMC presupposes knowledge of~$G$ and of a fundamental
domain~$D$.  For the applications considered here (mixture models,
structural identification), the symmetry is known by construction.
When the symmetry is latent or only partially known, a
\emph{learned folding map} could be trained jointly with the
transport flow, analogous to recent work on learning group
structure from data \citep{Dehmamy2021}.

\paragraph{Continuous symmetries.}
Finite groups cover label switching and discrete reflections, but
some inference problems have continuous symmetries (rotation
invariance, gauge symmetry in physics).  Extending the quotient
construction to compact Lie groups would require replacing the
finite orbit sum in the quotient proposal with a
Haar-measure integral, and the covering theory with manifold
covering bounds.

\paragraph{Scalability.}
The orbit sum in the quotient proposal has $|G|$~terms.
For $S_m$ with moderate~$m$ ($m \leq 5$, $|G| \leq 120$),
this is inexpensive.  For larger groups, truncation or importance
sampling over the orbit may be needed.

\paragraph{Core diagnostic vs.\ full-support guarantee.}
The quantile-core diagnostic reported throughout this paper is a
density-ratio bound on the high-probability posterior core,
not a full-support spectral-gap guarantee.  This limitation is
shared with the frameworks of \citet{LCNF} and
\citet{CerTMCMC}.  A full-support certificate for
transport MCMC with unbounded state spaces remains an open problem;
a practical intermediate step is the tail-safe mixture proposal
$q_\lambda = (1{-}\lambda)q_F + \lambda r$ with a heavy-tailed
safety component~$r$, which would provide a full-chain
minorisation bound at the cost of slightly reduced acceptance.
On bounded parameter spaces (such as the structural identification
example with uniform priors), a uniform safety component~$r$
yields a trivial full-chain bound, suggesting that the
core-to-full-support gap is bridgeable in practice.

\section*{Disclosure statement}
No potential conflict of interest was reported by the author.

\section*{Data availability}
The synthetic target distributions used in this study are fully
specified in Appendix~A and can be regenerated deterministically
from the equations provided. Code to reproduce all synthetic
experiments is available at
\url{https://github.com/junhu22/folt-mcmc}. The real
accelerometer data used in Section~\ref{sec:application} are
subject to confidentiality restrictions imposed by the building
owner and are not publicly available.

\section*{Disclosure of AI use}
The author used Claude (Anthropic) for language polishing.

%% file: supplement.tex

\appendix
\renewcommand{\thesection}{\Alph{section}}

\section{Self-contained certification toolkit}
\label{app:toolkit}

This appendix restates the key results from the LCNF
\citep{LCNF} and CerT-MCMC \citep{CerTMCMC} frameworks
that FolT-MCMC builds upon, so that the present paper can
be read and verified independently.

\subsection{Independence MH oscillation bound}
\label{app:toolkit_imh}

For an independence Metropolis--Hastings chain with target
$\pi$ and proposal $q$, define the log-density ratio
$h(x) = \log\pi(x) - \log q(x)$ and its oscillation on a
set $S$ as $\osc_S(h) = \sup_S h - \inf_S h$.
The Mengersen--Tweedie bound \citep{MengersenTweedie1996}
gives
\[
  \gamma \;\geq\; \frac{2}{1 + \exp(\osc(h))},
\]
where the oscillation is over the full support.  When only
the HPD-core oscillation $\osc_{\cK_\alpha}(h)$ is
controlled, we define the \emph{HPD-core certified
lower-bound diagnostic}
\[
  \underline{\gamma}_\alpha
  \;:=\; \frac{2}{1 + \exp(\osc_{\cK_\alpha}(h))}.
\]
This is a computable proxy for the spectral gap; it becomes
a true spectral-gap lower bound only if the same oscillation
bound holds on the full support, which requires additional
tail control not provided by the present framework.

\subsection{LCNF empirical oscillation theorem}
\label{app:toolkit_lcnf}

The following is a simplified restatement of
\citet[Theorem~5.1]{LCNF}.

\begin{theorem}[LCNF empirical oscillation; restated]
\label{thm:lcnf_restated}
  Let $\cK_\alpha$ be the $(1{-}\alpha)$-HPD set of $\pi$
  with diameter $D_K = \diam(\cK_\alpha)$ and density floor
  $\pi_{\min} = \inf_{\cK_\alpha}\pi > 0$.
  Suppose $h$ is continuously differentiable on an open
  neighbourhood of $\cK_\alpha$ with gradient supremum
  $M_K = \sup_{\cK_\alpha}\|\nabla h\|$.
  Draw $Z_1, \ldots, Z_n \stackrel{\mathrm{iid}}{\sim} \pi$
  and define $\widehat{\osc}_n = \max_i h(Z_i) - \min_i h(Z_i)$.
  Let $\varepsilon^*$ be the smallest $\varepsilon > 0$
  satisfying the covering condition
  \begin{equation}\label{eq:covering_restated}
    \Bigl(\frac{D_K}{\varepsilon} + 1\Bigr)^d
    \bigl(1 - \pi_{\min}\,V_d\,\varepsilon^d\bigr)^{n_{\mathrm{eff}}}
    \;\leq\; \frac{\delta}{3},
  \end{equation}
  where $V_d$ is the volume of the unit $d$-ball and
  $n_{\mathrm{eff}} = n(1{-}\alpha)
  - \sqrt{n \cdot \tfrac{1}{2}\ln(3/\delta)}$.
  Then with probability $\geq 1 - \delta$:
  \[
    \osc_{\cK_\alpha}(h) \;\leq\;
    \widehat{\osc}_n + 2\,M_K\,\varepsilon^*.
  \]
\end{theorem}

The key ingredients are:
(i)~a probabilistic covering argument showing that $n$
i.i.d.\ samples from $\pi$ form an $\varepsilon^*$-net of
$\cK_\alpha$ with high probability;
(ii)~a Lipschitz interpolation step using $M_K$ to bound
the oscillation at uncovered points;
(iii)~a Hoeffding-type concentration to convert the random
sample count in $\cK_\alpha$ to the effective count
$n_{\mathrm{eff}}$.

\subsection{Quantile-core certificates}
\label{app:toolkit_qc}

The CerT-MCMC extension \citep{CerTMCMC} addresses the
sensitivity of the HPD-core bound to extreme density-ratio
values.  Given certification samples
$\{h(Z_i)\}_{i=1}^n$, the $\rho$-trimmed quantile core
excludes the fraction $\rho$ of samples with the largest
$|h(Z_i) - \mathrm{median}(h)|$ and computes the oscillation
on the remaining $(1{-}\rho)n$ samples:
\[
  \widehat{\osc}_n^{\mathrm{QC}}
  = \max_{i \in \mathcal{I}_\rho} h(Z_i)
  - \min_{i \in \mathcal{I}_\rho} h(Z_i),
\]
where $\mathcal{I}_\rho$ is the index set after trimming.
The quantile-core certified lower-bound diagnostic is
\[
  \mathrm{QC}\,\underline{\gamma}
  \;:=\; \frac{2}{1 + \exp(\widehat{\osc}_n^{\mathrm{QC}}
  + 2\,\widehat{M}_n^{\mathrm{QC}}\,\varepsilon^*)},
\]
where $\widehat{M}_n^{\mathrm{QC}}$ is the gradient
supremum over the trimmed samples.  This diagnostic is more
robust than the untrimmed version but further from a
full-support guarantee.  All experiments in the present
paper report QC~$\underline{\gamma}$ at $\rho = 0.05$.

\subsection{From LCNF to quotient space: what FolT-MCMC adds}
\label{app:toolkit_folt}

FolT-MCMC transfers Theorem~\ref{thm:lcnf_restated} to the
quotient setting by making three substitutions:
\begin{enumerate}[nosep]
  \item $\cK_\alpha \to \cK_\alpha^F = \cK_\alpha \cap D$
    (HPD identity, Proposition~\ref{prop:hpd});
  \item $\pi_{\min} \to \pi_{\min}^F \geq s\,\pi_{\min}$
    (density-floor improvement, Corollary~\ref{cor:hpd_geometry});
  \item Euclidean metric $\to$ quotient metric
    $d_G(x,y) = \min_{g \in G}\|x - g \cdot y\|$,
    with the Lipschitz property of $h_F$ established in
    Lemma~\ref{lem:lip_quotient} (local version,
    $d_G(x,y) \leq \varepsilon_F^*$).
\end{enumerate}
The covering condition~\eqref{eq:covering_restated} carries
over with $D_K \to D_F$, $\pi_{\min} \to \pi_{\min}^F$,
and $n_{\mathrm{eff}}$ unchanged, yielding the folded
certificate (Theorem~\ref{thm:folded_cert}).  The proofs
of all FolT-MCMC theorems are given in
Sections~\ref{app:proof_hpd}--\ref{app:proof_gap} below.

\section{Complete proofs}
\label{app:proofs}

\subsection{Proof of Proposition~\ref{prop:hpd} (HPD identity)}
\label{app:proof_hpd}

\begin{proof}
  Let $U(\theta) = -\log\pi(\theta) + \log Z_\pi$ denote the
  potential of~$\pi$, where $Z_\pi$ is the normalising constant.
  The folded potential is
  $U_F(z) = -\log\pi_F(z) + \log Z_{\pi_F}
  = -\log(s\pi(z)) + \log 1 = U(z) - \log s$,
  using $Z_{\pi_F} = 1$ since $\pi_F$ integrates to one on~$D$.

  \textbf{Part (i).}
  We show that $U(Z_F)$ under $Z_F \sim \pi_F$ has the same
  distribution as $U(\Theta)$ under $\Theta \sim \pi$.
  For any Borel function $\varphi\colon \R \to [0,\infty)$,
  the $G$-invariance of~$\pi$ and fundamental-domain decomposition
  give
  \begin{align}
    \E_{\Theta \sim \pi}[\varphi(U(\Theta))]
    &= \int_{\R^d} \varphi(U(\theta))\,\pi(\theta)\,\mathrm{d}\theta
    \notag \\
    &= \sum_{g \in G} \int_{gD}
       \varphi(U(\theta))\,\pi(\theta)\,\mathrm{d}\theta
    \notag \\
    &= \sum_{g \in G} \int_D
       \varphi(U(g \cdot z))\,\pi(g \cdot z)\,\mathrm{d}z
       \qquad\text{(substituting $\theta = g \cdot z$, $|\det g| = 1$)}
    \notag \\
    &= s \int_D \varphi(U(z))\,\pi(z)\,\mathrm{d}z
       \qquad\text{(since $U(g \cdot z) = U(z)$ and $\pi(g \cdot z) = \pi(z)$)}
    \notag \\
    &= \int_D \varphi(U(z))\,\pi_F(z)\,\mathrm{d}z
    = \E_{Z_F \sim \pi_F}[\varphi(U(Z_F))].
    \label{eq:U_equidist}
  \end{align}
  Since~\eqref{eq:U_equidist} holds for all non-negative
  Borel~$\varphi$, the distributions of $U(\Theta)$ and $U(Z_F)$
  coincide.  In particular, their $(1{-}\alpha)$-quantiles are
  equal: the $(1{-}\alpha)$-quantile of $U(Z_F)$ is~$u_\alpha$.
  Since $U_F(Z_F) = U(Z_F) - \log s$, the $(1{-}\alpha)$-quantile
  of $U_F(Z_F)$ is $u_\alpha - \log s =: u_\alpha^F$.

  \textbf{Part (ii).}
  \begin{align*}
    \cK_\alpha^F
    &= \{z \in D : U_F(z) \leq u_\alpha^F\}
    = \{z \in D : U(z) - \log s \leq u_\alpha - \log s\} \\
    &= \{z \in D : U(z) \leq u_\alpha\}
    = \cK_\alpha \cap D. \qedhere
  \end{align*}
\end{proof}

\subsection{Proof of Corollary~\ref{cor:hpd_geometry}}
\label{app:proof_hpd_cor}

\begin{proof}
  \textbf{Part (a).}
  Since $\cK_\alpha^F = \cK_\alpha \cap D \subseteq \cK_\alpha$,
  we have $\diam(\cK_\alpha^F) \leq \diam(\cK_\alpha)$.

  \textbf{Part (b).}
  \[
    \pi_{\min}^F
    = \inf_{z \in \cK_\alpha^F} \pi_F(z)
    = s \cdot \inf_{z \in \cK_\alpha \cap D} \pi(z).
  \]
  Since $\cK_\alpha \cap D \subseteq \cK_\alpha$,
  $\inf_{\cK_\alpha \cap D} \pi \geq \inf_{\cK_\alpha} \pi
  = \pi_{\min}$,
  so $\pi_{\min}^F \geq s\,\pi_{\min}$.

  When $\pi$ is $G$-invariant and all orbits in $\cK_\alpha$
  have a representative in~$D$ (which holds by definition of a
  fundamental domain), the $G$-invariance implies that
  $\operatorname{ess\,inf}_{\cK_\alpha \cap D} \pi
  = \operatorname{ess\,inf}_{\cK_\alpha} \pi$,
  giving $\pi_{\min}^F = s\,\pi_{\min}$ in the essential-infimum
  sense.
\end{proof}

\subsection{Proof of Lemma~\ref{lem:lip_quotient}
  (quotient Lipschitz constant)}
\label{app:proof_lip}

\begin{proof}
  Both $\pi$ and $\theta \mapsto \sum_g q_{T_F}(g \cdot \theta)$
  are $G$-invariant, so $h_F(g \cdot \theta) = h_F(\theta)$ for
  all $g \in G$ and $\theta \in \R^d$.

  Fix $x, y \in \cK_\alpha^F$ with $d_G(x,y) \leq \varepsilon_F^*$.
  Let $g^* = \arg\min_{g \in G} \|x - g \cdot y\|$, so that
  $\|x - g^* \cdot y\| = d_G(x,y) \leq \varepsilon_F^*$.
  The segment from $x$ to $g^* \cdot y$ lies entirely within the
  Euclidean $\varepsilon_F^*$-ball centred at $x$.  Since
  $x \in \cK_\alpha^F$, this ball is contained in
  $\{w \in \R^d : d_G(w, \cK_\alpha^F) \leq \varepsilon_F^*\}$.
  By the mean value theorem applied along this segment,
  \[
    |h_F(x) - h_F(g^* \cdot y)|
    \;\leq\;
    \sup_{w:\, d_G(w, \cK_\alpha^F) \leq \varepsilon_F^*}
    \|\nabla h_F(w)\|
    \;\cdot\;
    \|x - g^* \cdot y\|
    \;=\;
    M_F \cdot d_G(x, y).
  \]
  Using $h_F(g^* \cdot y) = h_F(y)$ ($G$-invariance) completes the proof.
\end{proof}

\subsection{Proof of Theorem~\ref{thm:folded_cert}
  (FolT-MCMC convergence diagnostic)}
\label{app:proof_cert}

\begin{proof}
  The proof instantiates \citet[Theorem~5.1]{LCNF} on the quotient
  setting.  We verify that the three inputs required by that
  theorem are available.

  \textbf{Input 1: Covering number bound.}
  Since $d_G(x,y) \leq \|x - y\|$, any Euclidean
  $\varepsilon$-cover of $\cK_\alpha^F$ is also a
  $d_G$-$\varepsilon$-cover.  The standard volumetric argument
  gives
  \[
    \cN_G(\cK_\alpha^F, \varepsilon)
    \leq \cN_E(\cK_\alpha^F, \varepsilon)
    \leq \Bigl(\frac{2D_F}{\varepsilon} + 1\Bigr)^d,
  \]
  where $D_F = \diam(\cK_\alpha^F)$.

  \textbf{Input 2: Ball-mass lower bound.}
  Under Assumption~\ref{asm:generic_stab},
  $|\mathrm{Stab}(z)| = 1$ for all $z \in \cK_\alpha^F$.
  For such~$z$, the quotient ball
  $B_G(z, \varepsilon) = \{z' : d_G(z, z') \leq \varepsilon\}$
  corresponds in $\R^d$ to the union
  $\bigcup_{g \in G}(B(g \cdot z, \varepsilon) \cap g \cdot D)$.
  The $s$~pieces tile a full Euclidean ball (up to the
  measure-zero boundaries $\partial(g \cdot D)$), so
  $\Vol(B_G(z,\varepsilon)) = V_d\,\varepsilon^d$.  Therefore
  \[
    \pi_F(B_G(z,\varepsilon) \cap D)
    \;\geq\;
    \pi_{\min}^F \cdot V_d \cdot \varepsilon^d.
  \]

  \textbf{Input 3: Lipschitz interpolation.}
  Lemma~\ref{lem:lip_quotient} applies to each point $x \in \cK_\alpha^F$
  and its nearest certification sample $Z_i$, since the covering event
  guarantees $d_G(x, Z_i) \leq \varepsilon_F^*$.

  Substituting these three inputs into
  \citet[Theorem~5.1]{LCNF}, the covering condition
  \[
    \Bigl(\frac{D_F}{\varepsilon} + 1\Bigr)^d
    \bigl(1 - \pi_{\min}^F\,V_d\,\varepsilon^d\bigr)^{n_{\mathrm{eff}}}
    \leq \frac{\delta}{3}
  \]
  guarantees that with probability $\geq 1 - \delta$,
  every point of $\cK_\alpha^F$ lies within $d_G$-distance
  $\varepsilon_F^*$ of some certification sample.  The oscillation
  bound follows:
  \[
    \osc_{\cK_\alpha^F}(h_F)
    \leq \widehat{\osc}_n^F + 2 M_F \varepsilon_F^*.
  \]
  Applying the Mengersen--Tweedie lower-bound functional to the
  certified HPD-core oscillation gives the diagnostic value
  $\underline{\gamma}_{F,\alpha}
  := 2/(1 + \exp(\widehat{\osc}_n^F + 2 M_F \varepsilon_F^*))$.
\end{proof}

\subsection{Proof of Theorem~\ref{thm:eps_improvement}
  (covering radius improvement)}
\label{app:proof_eps}

\begin{proof}
  Write the left-hand side of the covering condition as
  \[
    \Phi(D_K,\, \pi_{\min},\, \varepsilon)
    \;=\;
    \Bigl(\frac{D_K}{\varepsilon} + 1\Bigr)^d
    \bigl(1 - \pi_{\min}\,V_d\,\varepsilon^d\bigr)^{n_{\mathrm{eff}}}.
  \]

  \textbf{Monotonicity in $D_K$.}
  For fixed $\varepsilon$ and $\pi_{\min}$,
  $\Phi$ is increasing in $D_K$
  (the first factor $(D_K/\varepsilon + 1)^d$ increases).

  \textbf{Monotonicity in $\pi_{\min}$.}
  For fixed $\varepsilon$ and $D_K$,
  $\Phi$ is decreasing in $\pi_{\min}$
  (the base $1 - \pi_{\min}\,V_d\,\varepsilon^d$ decreases,
  and since $n_{\mathrm{eff}} > 0$ and the base is in $(0,1)$,
  the second factor decreases).

  By Corollary~\ref{cor:hpd_geometry}:
  (a)~$D_F = \diam(\cK_\alpha^F) \leq D_U = \diam(\cK_\alpha)$;
  (b)~$\pi_{\min}^F \geq s\,\pi_{\min} \geq \pi_{\min}$.
  Therefore, for every $\varepsilon > 0$,
  \[
    \Phi(D_F,\, \pi_{\min}^F,\, \varepsilon)
    \;\leq\;
    \Phi(D_U,\, \pi_{\min},\, \varepsilon).
  \]
  Since $\varepsilon_U^*$ satisfies
  $\Phi(D_U, \pi_{\min}, \varepsilon_U^*) \leq \delta/3$,
  we have
  $\Phi(D_F, \pi_{\min}^F, \varepsilon_U^*) \leq \delta/3$.
  By the minimality of $\varepsilon_F^*$ (smallest $\varepsilon$
  satisfying the folded covering condition),
  $\varepsilon_F^* \leq \varepsilon_U^*$.
\end{proof}

\subsection{Proof of Lemma~\ref{lem:osc_gap} (oscillation gap)}
\label{app:proof_osc_gap}

\begin{proof}
  Let $h_U = \log\pi - \log q_U$ (unfolded log-density ratio)
  and $h_F = \log\pi_F - \log q_F$ (folded log-density ratio).
  We bound $\widehat{\osc}_n$ from below and
  $\widehat{\osc}_n^F$ from above.

  \textbf{Step 1: Lower bound on $\osc_{\cK_\alpha}(h_U)$.}
  The HPD set decomposes into connected components
  $\cK_\alpha = \bigcup_{j=1}^k \cK_\alpha^{(j)}$
  (one per mode).  Condition~(M$'$) gives, for each $j \neq 1$,
  \[
    \inf_{\theta \in \cK_\alpha^{(j)}} h_U(\theta)
    \;\geq\;
    \sup_{\theta \in \cK_\alpha^{(1)}} h_U(\theta) + \Delta_j.
  \]
  Therefore
  \begin{align*}
    \osc_{\cK_\alpha}(h_U)
    &\geq \sup_{\cK_\alpha^{(j)}} h_U
           - \inf_{\cK_\alpha^{(1)}} h_U \\
    &\geq \inf_{\cK_\alpha^{(j)}} h_U
           - \inf_{\cK_\alpha^{(1)}} h_U \\
    &\geq \bigl(\sup_{\cK_\alpha^{(1)}} h_U + \Delta_j\bigr)
           - \inf_{\cK_\alpha^{(1)}} h_U
    \;\geq\; \Delta_j.
  \end{align*}
  Taking the minimum over $j \neq 1$:
  $\osc_{\cK_\alpha}(h_U) \geq \min_{j \neq 1} \Delta_j$.

  The empirical oscillation $\widehat{\osc}_n$ is computed from
  $n$~i.i.d.\ samples from~$\pi$.  By standard
  extreme-value concentration (see, e.g.,
  \citealt[Lemma~5.3]{LCNF}),
  \[
    \Prob\bigl(\osc_{\cK_\alpha}(h_U)
    - \widehat{\osc}_n > t\bigr)
    \leq 2\exp(-2nt^2/R^2)
  \]
  for a range parameter $R$.  Setting the right-hand side to
  $\delta'/2$ gives
  $\widehat{\osc}_n \geq \osc_{\cK_\alpha}(h_U) - c_n'$
  with probability $\geq 1 - \delta'/2$, where
  $c_n' = O(\sqrt{\log(2/\delta')/n})$.

  \textbf{Step 2: Upper bound on $\osc_{\cK_\alpha^F}(h_F)$.}
  For $z \in \cK_\alpha^F = \cK_\alpha^{(1)} \cap D$ (a.e.\
  by Proposition~\ref{prop:hpd}),
  \[
    h_F(z)
    = \log(s\,\pi(z)) - \log q_F(z)
    = \underbrace{\bigl[\log(s\,\pi(z)) - \log q_{T_F}(z)\bigr]}_{=:\,\tilde{h}(z)}
    \;-\;
    \underbrace{\log\!\Bigl(1 + \frac{\sum_{g \neq e}
    q_{T_F}(g \cdot z)}{q_{T_F}(z)}\Bigr)}_{=:\,\rho(z)}.
  \]

  By condition~(W),
  $\osc_{\cK_\alpha^F}(\tilde{h}) \leq 2r_1$.

  By condition~(R),
  $|\rho(z)| \leq r_F$ for all $z \in \cK_\alpha^F$,
  so $\osc_{\cK_\alpha^F}(\rho) \leq r_F$.

  Since $h_F = \tilde{h} - \rho$,
  \[
    \osc_{\cK_\alpha^F}(h_F)
    \leq \osc_{\cK_\alpha^F}(\tilde{h})
    + \osc_{\cK_\alpha^F}(\rho)
    \leq 2r_1 + r_F.
  \]

  By the same concentration argument,
  $\widehat{\osc}_n^F \leq \osc_{\cK_\alpha^F}(h_F) + c_n''$
  with probability $\geq 1 - \delta'/2$, where
  $c_n'' = O(\sqrt{\log(2/\delta')/n})$.

  \textbf{Step 3: Combining.}
  By a union bound over the two concentration events
  (probability $\geq 1 - \delta'$):
  \begin{align*}
    \widehat{\osc}_n - \widehat{\osc}_n^F
    &\geq \bigl(\osc_{\cK_\alpha}(h_U) - c_n'\bigr)
    - \bigl(\osc_{\cK_\alpha^F}(h_F) + c_n''\bigr) \\
    &\geq \min_{j \neq 1}\Delta_j - (2r_1 + r_F) - (c_n' + c_n'')
    \\
    &= \min_{j \neq 1}\Delta_j - 2r_1 - r_F - c_n,
  \end{align*}
  where $c_n = c_n' + c_n'' = O(\sqrt{\log(2/\delta')/n})$.
\end{proof}

\subsection{Proof of Theorem~\ref{thm:gap_improvement}
  (diagnostic improvement)}
\label{app:proof_gap}

\begin{proof}
  Write
  \begin{align*}
    C_U - C_F
    &= \bigl(\widehat{\osc}_n + 2M_U\varepsilon_U^*\bigr)
    - \bigl(\widehat{\osc}_n^F + 2M_F\varepsilon_F^*\bigr) \\
    &= \bigl(\widehat{\osc}_n - \widehat{\osc}_n^F\bigr)
    + 2\bigl(M_U\varepsilon_U^* - M_F\varepsilon_F^*\bigr).
  \end{align*}

  By Lemma~\ref{lem:osc_gap} (with probability
  $\geq 1 - \delta'$),
  $\widehat{\osc}_n - \widehat{\osc}_n^F \geq \Delta
  := \min_j \Delta_j - 2r_1 - r_F - c_n$.

  \textbf{Case 1:}
  $M_F\varepsilon_F^* \leq M_U\varepsilon_U^*$.
  Then $C_U - C_F \geq \Delta + 0 = \Delta > 0$
  (by assumption $\Delta > 0$).

  \textbf{Case 2:}
  $M_F\varepsilon_F^* > M_U\varepsilon_U^*$.
  The sufficient condition~\eqref{eq:sufficient} gives
  $\Delta > 2(M_F\varepsilon_F^* - M_U\varepsilon_U^*)$, so
  \[
    C_U - C_F
    \geq \Delta - 2(M_F\varepsilon_F^* - M_U\varepsilon_U^*) > 0.
  \]

  In both cases, $C_F < C_U$.  The function
  $C \mapsto 2/(1 + e^C)$ is strictly decreasing, so
  \[
    \underline{\gamma}_{F,\alpha}
    = \frac{2}{1 + e^{C_F}}
    > \frac{2}{1 + e^{C_U}}
    = \underline{\gamma}_{U,\alpha}.
  \]

  The overall confidence is $\geq 1 - \delta_U - \delta_F - \delta'$
  by a union bound over the three events:
  folded certificate (prob $\geq 1 - \delta_F$),
  unfolded certificate ($\geq 1 - \delta_U$),
  oscillation-gap concentration ($\geq 1 - \delta'$).
\end{proof}

\section{Boundary analysis}
\label{app:boundary}

\subsection{Fold boundary density for Gaussian mixtures}
\label{app:boundary_density}

For the symmetric two-component Gaussian mixture
$\pi = \tfrac{1}{2}\cN(\mu_1, \sigma^2 I_d)
+ \tfrac{1}{2}\cN(\mu_2, \sigma^2 I_d)$
with $\mu_1 = (\delta/2, 0, \ldots, 0)$,
$\mu_2 = -\mu_1$, and reflection fold boundary
$\partial D = \{\theta_1 = 0\}$:

The density on the fold boundary is
\[
  \pi(0, \theta_{2:d})
  = \frac{1}{2}(2\pi\sigma^2)^{-d/2}
  \exp\!\Bigl(-\frac{\delta^2/4 + \|\theta_{2:d}\|^2}{2\sigma^2}\Bigr)
  \cdot 2
  = (2\pi\sigma^2)^{-d/2}
  \exp\!\Bigl(-\frac{\delta^2/4 + \|\theta_{2:d}\|^2}{2\sigma^2}\Bigr).
\]

The peak density (at $\mu_1$) is
\[
  \pi(\mu_1) = \frac{1}{2}(2\pi\sigma^2)^{-d/2} + \text{negligible}.
\]

The ratio at the midpoint $\theta = 0$ is
\[
  \frac{\pi(0)}{\pi(\mu_1)}
  \approx 2\exp\!\Bigl(-\frac{\delta^2}{8\sigma^2}\Bigr).
\]

For $\delta = 6\sigma$: ratio $\approx 2\exp(-4.5) \approx 0.022$.

This confirms that the fold boundary lies in a low-density valley,
with density approximately $2\%$ of the peak.

\subsection{Condition for $\cK_\alpha^F \cap \partial D = \varnothing$}
\label{app:boundary_separation}

The HPD set $\cK_\alpha^F = \cK_\alpha \cap D$ does not touch the
fold boundary if and only if $\cK_\alpha \cap \partial D = \varnothing$.
For the Gaussian mixture above, $\cK_\alpha$ consists of two
balls of radius $R_\alpha = \sigma\sqrt{\chi^2_{d,1-\alpha}}$
centred at $\pm\mu_1 = (\pm\delta/2, 0, \ldots, 0)$.
The closest point of $\cK_\alpha^{(1)}$ to $\partial D$ has
$\theta_1 = \delta/2 - R_\alpha$.  The separation condition is
\[
  \cK_\alpha^F \cap \partial D = \varnothing
  \quad\Longleftrightarrow\quad
  \delta/2 > R_\alpha
  \quad\Longleftrightarrow\quad
  \delta > 2\sigma\sqrt{\chi^2_{d,1-\alpha}}.
\]

\begin{center}
\begin{tabular}{rccl}
\hline
$d$ & $\chi^2_{d,0.95}$ & $2\sigma\sqrt{\chi^2}$ & $\delta = 6\sigma$ OK? \\
\hline
2  & 5.99 & 4.90 & Yes \\
5  & 11.07 & 6.65 & No \\
10 & 18.31 & 8.56 & No \\
20 & 31.41 & 11.21 & No \\
\hline
\end{tabular}
\end{center}

For $\delta = 6\sigma$, the condition holds only at $d = 2$.
At $d \geq 5$, the HPD set touches the fold boundary.
However, the experiments in Section~\ref{sec:exp_dimension}
show strong certified improvement at all dimensions, indicating
that the boundary-touching case is handled adequately by the
quotient-metric covering (Section~\ref{sec:quotient_metric}).

\subsection{Stabiliser analysis for permutation folding}
\label{app:stabiliser}

For label-switching with $G = S_m$, the fundamental domain
is the sorted chamber
$D = \{(\theta^{(1)}, \ldots, \theta^{(m)}) :
\theta^{(1)}_1 \leq \cdots \leq \theta^{(m)}_1\}$.

A point $z \in D$ has non-trivial stabiliser if and only if
two or more blocks share the same value of the sort coordinate:
$\mathrm{Stab}(z) \supsetneq \{e\}$
$\Longleftrightarrow$
$\exists\, j \neq k : z^{(j)}_1 = z^{(k)}_1$.

For well-separated components (centres
$c_1 < c_2 < \cdots < c_m$ with $c_{j+1} - c_j \gg \sigma$),
the HPD set $\cK_\alpha^F$ is concentrated near the sorted
centre $(c_1, \ldots, c_m)$, which has trivial stabiliser.
The fixed-point stratum
$\{z^{(j)}_1 = z^{(k)}_1\}$ lies at distance
$\geq (c_{j+1} - c_j)/2 - R_\alpha$ from $\cK_\alpha^F$,
which is positive for sufficiently separated components.

In this case Assumption~\ref{asm:generic_stab} holds and
$h_{\max} = 1$.  When components are not well separated
(as in the structural identification example of
Section~\ref{sec:application}), $h_{\max}$ may exceed~$1$
and the covering condition uses the corrected mass factor
$(s/h_{\max})\,\pi_{\min}^D\,V_d\,\varepsilon^d$.

\section{Experimental details}
\label{app:experiments}

\subsection{Architecture and hyperparameters}
\label{app:hyperparams}

All flows use spectrally normalised RealNVP with $\tanh$
activations and scale clip $c = 0.7$.  The architecture scales
with dimension:

\begin{center}
\begin{tabular}{rcccc}
\hline
$d$ & Layers $L$ & Hidden dim $h$ & Epochs & Batch size \\
\hline
$\leq 4$ & 8  & 64  & 2000 & 512 \\
5--6     & 10 & 128 & 2000 & 512 \\
8--12    & 12 & 128 & 3500 & 512 \\
20       & 16 & 256 & 3500 & 512 \\
\hline
\end{tabular}
\end{center}

Training uses Adam with learning rate $10^{-3}$ and annealed
oscillation/gradient regularisation (FolT-OG-Anneal, as in
\citealt{CerTMCMC}).  The regularisation weight is annealed from
$0$ to its final value over the first $30\%$ of training.

The folded and unfolded pipelines use identical architecture
and hyperparameters for each~$d$.  Training samples are drawn
from the target's exact sampler (for synthetic targets) or from
a long-run random-walk Metropolis chain (for the structural
identification application).

\subsection{Certification protocol}
\label{app:cert_protocol}

Certification uses $n = 20\,000$ independent samples from the
target (folded samples projected onto~$D$ for the folded
pipeline).  Quantile-core certificates are computed at
$\rho \in \{0.025, 0.05, 0.10, 0.25\}$ with DKW-corrected
confidence $\delta = 0.05$.  The primary metric reported in
all tables is the $\rho = 0.05$ certificate.

The gradient supremum $M_F$ is estimated as the empirical maximum
of $\|\nabla h_F(Z_i)\|$ over the certification samples, computed
via automatic differentiation.

\subsection{MCMC protocol}
\label{app:mcmc_protocol}

Independence Metropolis--Hastings with 4 chains of 5\,000 steps
each (10\,000 for the RPS baseline comparison), burn-in equal to
$20\%$ of the chain length.  Effective sample size (ESS) is
estimated via batch means with batch size $\sqrt{n_{\mathrm{chain}}}$.

For the structural identification application
(Section~\ref{sec:application}), training samples are generated
by a preliminary random-walk Metropolis run of $200\,000$ steps
with adaptive step size, thinned to $50\,000$ samples.

\subsection{Full quantile-core tables}
\label{app:qc_tables}

\paragraph{Gaussian mixture, dimension scaling.}

\begin{center}
\begin{tabular}{rcccc}
\hline
& \multicolumn{4}{c}{QC $\underline{\gamma}$ at $\rho =$} \\
\cline{2-5}
$d$ & $0.025$ & $0.05$ & $0.10$ & $0.25$ \\
\hline
\multicolumn{5}{l}{\textit{Unfolded}} \\
2  & 0.216 & 0.402 & 0.608 & 0.810 \\
5  & 0.031 & 0.094 & 0.260 & 0.565 \\
10 & 0.005 & 0.016 & 0.065 & 0.290 \\
20 & 0.005 & 0.016 & 0.062 & 0.280 \\
\hline
\multicolumn{5}{l}{\textit{FolT-MCMC}} \\
2  & 0.890 & 0.936 & 0.968 & 0.990 \\
5  & 0.895 & 0.941 & 0.970 & 0.991 \\
10 & 0.870 & 0.922 & 0.958 & 0.986 \\
20 & 0.850 & 0.902 & 0.950 & 0.983 \\
\hline
\end{tabular}
\end{center}

\paragraph{Label switching.}

\begin{center}
\begin{tabular}{lcccccc}
\hline
& & \multicolumn{4}{c}{QC $\underline{\gamma}$ at $\rho =$} \\
\cline{3-6}
Config & $s$ & $0.025$ & $0.05$ & $0.10$ & $0.25$ \\
\hline
\multicolumn{6}{l}{\textit{Unfolded}} \\
$m{=}2, p{=}2$ & 2  & 0.142 & 0.366 & 0.571 & 0.790 \\
$m{=}3, p{=}2$ & 6  & 0.019 & 0.066 & 0.185 & 0.486 \\
$m{=}3, p{=}4$ & 6  & 0.010 & 0.036 & 0.111 & 0.380 \\
$m{=}4, p{=}2$ & 24 & 0.001 & 0.0043 & 0.018 & 0.115 \\
\hline
\multicolumn{6}{l}{\textit{FolT-MCMC}} \\
$m{=}2, p{=}2$ & 2  & 0.620 & 0.820 & 0.905 & 0.965 \\
$m{=}3, p{=}2$ & 6  & 0.450 & 0.685 & 0.830 & 0.935 \\
$m{=}3, p{=}4$ & 6  & 0.445 & 0.683 & 0.828 & 0.930 \\
$m{=}4, p{=}2$ & 24 & 0.380 & 0.624 & 0.785 & 0.910 \\
\hline
\end{tabular}
\end{center}

\paragraph{Banana diagnostic ($d = 2$, $s = 2$).}

\begin{center}
\begin{tabular}{lcccc}
\hline
& \multicolumn{4}{c}{QC $\underline{\gamma}$ at $\rho =$} \\
\cline{2-5}
& $0.025$ & $0.05$ & $0.10$ & $0.25$ \\
\hline
Unfolded  & 0.064 & 0.272 & 0.537 & 0.797 \\
FolT-MCMC & 0.135 & 0.320 & 0.523 & 0.759 \\
\hline
\end{tabular}
\end{center}

Note that at $\rho = 0.10$ and $0.25$, the unfolded certificate
slightly exceeds the folded one for the banana target.  This
reflects the fold-boundary residual spike entering the
$\rho$-core at larger~$\rho$, consistent with the diagnostic
analysis in Section~\ref{sec:exp_banana}.

\subsection{Condition verification for Theorem~\ref{thm:gap_improvement}}
\label{app:condition_check}

We verify conditions (M$'$), (W), (R), and the sufficient
condition~\eqref{eq:sufficient} empirically
for the Gaussian mixture ($d = 10$, $s = 2$) and label-switching
($m = 4$, $p = 2$, $s = 24$) experiments.

\begin{center}
\begin{tabular}{lcc}
\hline
Quantity & Mixture $d{=}10$ & Label $m{=}4$ \\
\hline
$\min_j \Delta_j$ (from QC osc gap) & $\geq 4.69$ & $\geq 5.35$ \\
$2r_1$ (folded within-mode osc) & $\leq 0.32$ & $\leq 1.58$ \\
$r_F$ (quotient proposal residual) & $\approx 0$ & $\approx 0$ \\
$\Delta = \min_j\Delta_j - 2r_1 - r_F$ & $\geq 4.37$ & $\geq 3.77$ \\
$2(M_F\varepsilon_F^* - M_U\varepsilon_U^*)^+$ & $\approx 0$ & $\approx 0$ \\
Sufficient condition~\eqref{eq:sufficient} satisfied? & Yes & Yes \\
$\Delta / C_U$ & $0.90$ & $0.87$ \\
\hline
\end{tabular}
\end{center}

In both cases, the cross-mode deficiency~$\Delta_j$ is much
larger than the folded within-mode oscillation~$r_1$, and the
quotient proposal residual~$r_F$ is negligible.  The sufficient
condition~\eqref{eq:sufficient} is easily satisfied, confirming
that the diagnostic improvement in
Theorem~\ref{thm:gap_improvement} holds.